\documentclass{article}
\usepackage{iclr2027_conference,times}
\iclrfinalcopy
\usepackage[american]{babel}
\usepackage{mathtools,amssymb,amsthm}
\usepackage{graphicx,tikz,microtype}
\usepackage[caption=false,font=footnotesize]{subfig}
\usepackage[ruled,vlined]{algorithm2e}
\usepackage{enumitem,booktabs,tabularx}
\usepackage[hidelinks]{hyperref}
\usepackage{xurl}
\hypersetup{
  pdftitle={Hold-Out Scoring for Efficient Gaussian DAG Learning},
  pdfauthor={Donguk Shin, Byeongguk Kang, Inseol Lee, Gunwoong Park},
  pdflang={en-US}
}
\newtheorem{theorem}{Theorem}[section]
\newtheorem{lemma}[theorem]{Lemma}
\newtheorem{proposition}[theorem]{Proposition}

\theoremstyle{definition}
\newtheorem{assumption}{Assumption}

\theoremstyle{remark}

\DeclareMathOperator*{\argmin}{argmin}
\DeclareMathOperator{\clip}{clip}
\DeclareMathOperator{\diag}{diag}
\DeclareMathOperator{\rank}{rank}
\DeclareMathOperator{\Var}{Var}
\DeclareMathOperator{\Cov}{Cov}
\newcommand{\pa}{\operatorname{Pa}}
\newcommand{\an}{\operatorname{An}}
\newcommand{\Src}{\operatorname{Src}}
\newcommand{\Nsrc}{\operatorname{Nsrc}}
\newcommand{\E}{\mathbb E}
\renewcommand{\P}{\mathbb P}
\newcommand{\R}{\mathbb R}

\newcommand{\pathto}{\rightsquigarrow}
\title{Hold-Out Scoring for Efficient Gaussian DAG Learning}
\author{
Donguk Shin\textsuperscript{1},
Byeongguk Kang\textsuperscript{2},
Inseol Lee\textsuperscript{1},
Gunwoong Park\textsuperscript{1,2,3}\thanks{Corresponding author: \texttt{gw.park23@gmail.com}.}
}
\begin{document}
\maketitle
\lhead{}
\footnotetext[1]{Department of Statistics, Seoul National University.}
\footnotetext[2]{Interdisciplinary Program in Artificial Intelligence, Seoul National University.}
\footnotetext[3]{Institute for Data Innovation in Science, Seoul National University.}
\allowdisplaybreaks[1]

\begin{abstract}
High-dimensional Gaussian DAG learning faces a statistical--computational gap: methods with sharp sample complexity rely on computationally expensive subset search and a supplied indegree bound, whereas polynomial-time alternatives have less favorable sample complexity. We introduce HOST, an efficient DAG learning algorithm that replaces subset search with nodewise hold-out scoring and convex regression, without requiring a supplied indegree bound. Our key insight is that recovering a correct ordering does not require uniformly small estimation errors in ordering scores but only one-sided control of those errors. In the ordering step, HOST exploits the fact that score estimation using hold-out samples inflates ordering scores in expectation, which is the favorable direction for candidates that should not yet be selected. Given the ordering, HOST recovers parents by recursively removing indirect effects from total effects between two nodes. Under suitable conditions, HOST exactly recovers a $p$-node DAG of maximum indegree $d$ with sample complexity of order $d\log p$ in polynomial time. Experiments show that HOST achieves competitive graph recovery while exhibiting favorable runtime scaling.
\end{abstract}

\section{Introduction}

Learning a directed acyclic graph (DAG) from observational data is a central problem in graphical modeling and causal discovery \citep{pearl1988probabilistic, spirtes2001causation}. Classical approaches include constraint-based and score-based methods \citep{spirtes2001causation, chickering2002optimal}. These approaches generally identify the graph only up to a Markov equivalence class. Additional assumptions on the data-generating process, such as linear non-Gaussian models or nonlinear additive-noise structure, can make causal directions identifiable within a Markov equivalence class \citep{shimizu2006linear, shimizu2011directlingam, hoyer2008nonlinear, peters2014causal}. Other DAG-learning approaches include ordering-based methods that separate order search from edge selection \citep{buhlmann2014cam} and continuous optimization using a smooth acyclicity constraint \citep{zheng2018dags}. More recent methods use derivatives of the log density for causal ordering and, in some cases, edge selection \citep{rolland2022score, montagna2023arbitrary, montagna2023scalable, xu2024ordering}.

We focus on Gaussian linear DAG models, in which each node is expressed as a linear function of its parents plus an independent Gaussian error. Suitable conditions involving direct effects and error variances make the DAG identifiable from the observational distribution \citep{peters2014identifiability, park2020identifiability, park2020conditional, chen2019causal, ghoshal2018learning}. With identifiability established, we ask whether Gaussian linear DAGs can be learned with both sample efficiency and computational tractability. 

Existing methods exhibit a statistical--computational tradeoff. Forward methods, which iteratively search for nodes appearing early in an ordering, can achieve sharp sample complexity in terms of the maximum indegree $d$. For example, \citet{gao2022optimal} develop an algorithm that attains the information-theoretically optimal sample complexity $n \asymp d\log(p/d)$ under suitable conditions, where $n$ is the sample size and $p$ denotes the number of variables. However, several forward methods require a supplied indegree bound $q\ge d$ and costly subset search with complexity $p^{O(q)}$, which becomes prohibitive as $q$ grows \citep{chen2019causal,gao2022optimal,daskalakis2025condition}.

Backward procedures recursively identify and eliminate nodes from the end of an ordering using sparse inverse-covariance or regression methods \citep{ghoshal2018learning,park2021learning2}. These methods are computationally tractable, as they require sequential convex optimization, but their statistical guarantees depend on the maximum moralized degree $d_{\mathrm m}$. The maximum moralized degree can substantially exceed the maximum indegree $d$ because it also accounts for children and co-parents. These results motivate a procedure that replaces subset search with convex optimization, removes the need for a supplied indegree bound, and achieves sample complexity with a mild dependence on $d$ rather than the potentially much larger $d_{\mathrm m}$.

\begin{table}
\centering
\caption{Comparison of Gaussian DAG learning algorithms. The results for \citet{chen2019causal} concern ordering recovery only. In this table, $\gtrsim$ suppresses constants, including model-dependent factors. $\widetilde O$ suppresses logarithms. Irrep. denotes irrepresentability. Input $q$ indicates whether a supplied indegree bound $q\ge d$ is required.}
\label{tab:method-comparison}
\small
\begin{tabularx}{\linewidth}{@{}>{\raggedright\arraybackslash}p{0.30\linewidth}>{\centering\arraybackslash}Xccc@{}}
\toprule
Reference & Sample complexity & Computational complexity & Irrep. & Input $q$ \\
\midrule
\citet{ghoshal2018learning}
& $n\gtrsim d_{\mathrm m}^4\log p$
& $\widetilde O(p^3+pd_{\mathrm m}^4)$ & $\times$ & $\times$ \\
\citet{chen2019causal}
& $n\gtrsim d^2\log p$
& $p^{O(q)}$ & $\times$ & $\checkmark$ \\
\citet{park2021learning2}
& $n\gtrsim d_{\mathrm m}^2\log p$
& $O\{n(p^3+p^2d_{\mathrm m}^2)\}$ & $\checkmark$ & $\times$ \\
\citet{gao2022optimal}
& $n\gtrsim d\log(p/d)$
& $p^{O(q)}$ & $\times$ & $\checkmark$ \\
\citet{daskalakis2025condition}, Alg.~1
& $n\gtrsim d\log(p/d)$
& $p^{O(q)}$ & $\times$ & $\checkmark$ \\
\citet{daskalakis2025condition}, Alg.~2/3
& $n\gtrsim d^4\log p$
& $\operatorname{poly}(n,p)$ & $\times$ & $\checkmark$ \\
\textsc{HOST}
& $n\gtrsim d\log p$
& $O(np^3+p^4)$ & $\times$ & $\times$ \\
\bottomrule
\end{tabularx}
\end{table}

The need for subset search with a supplied indegree bound in forward methods arises because forward ordering recovery repeatedly compares scores to determine which remaining node has a sufficiently small score to be selected next. When the same observations are used for fitting and scoring, in-sample overfitting can produce artificially small scores and cause a node with unselected parents to be selected too early. To control this error uniformly across candidates, existing subset-search methods restrict the candidate models using a supplied indegree bound.

Our key observation is that accurate estimation of every score is unnecessary since recovering a compatible ordering requires only one-sided score control. We control overestimation for nodes whose parents have all been selected and underestimation for nodes with unselected parents. On held-out data, fitting error inflates scores in expectation, which is the favorable direction for nodes that should not yet be selected. Accurate estimation of scores is therefore needed only for nodes whose parents have all been selected. We exploit this asymmetry through \textsc{HOST} (Hold-Out Ordering and Structural Thresholding), which combines nodewise hold-out scoring for ordering recovery with recursive direct-effect reconstruction for parent recovery. Our main contributions are as follows.
\begin{itemize}[leftmargin=1.1em,labelsep=0.4em]
\item \textbf{Ordering recovery without a supplied indegree bound.} HOST uses Lasso-based regressions with a two-fold sample split to control score underestimation without dependence on the regression dimension, while accurate estimation of scores is required only for nodes whose parents have all been selected. This guarantees ordering recovery with $d$-dependent sample complexity.
\item \textbf{Parent recovery without irrepresentability.} Given the recovered ordering, HOST reconstructs direct effects from total effects by processing earlier nodes in reverse order and subtracting indirect effects. The ordering fits enter only through bounded corrections, allowing one threshold per node to recover parents without exact Lasso support recovery or an irrepresentability condition, even when the ordering fits omit true parents or direct and indirect effects cancel in the total effect.
\item \textbf{Near-optimal sample complexity in polynomial time.} Under explicit model conditions, HOST recovers the DAG with sample complexity of order $d\log p$ and takes polynomial time without subset search. Under the fixed-parameter convention of \citet{gao2022optimal}, this bound is within a factor of order $\log d$ of the minimax-optimal sample complexity $d\log(p/d)$ for $2\le d\le p/2$.
\end{itemize}

\section{Preliminaries}
\label{SecProb}\label{sec:model}
\subsection{Notation}
For a vector $v\in\mathbb R^m$, a matrix $M\in\mathbb R^{m\times n}$, and index sets $I\subseteq[m]$ and $J\subseteq[n]$, let $v_I$ and $M_{IJ}$ denote the corresponding subvector and submatrix with singleton indices suppressed. We write $\|\cdot\|_1$ and $\|\cdot\|=\|\cdot\|_2$ for the $\ell_1$ and Euclidean norms, respectively. We use $I_k$ for the $k\times k$ identity matrix and $\lambda_{\min}(M)$ for the smallest eigenvalue when $M$ is symmetric. $a\gtrsim b$ and $a\asymp b$ denote inequality and equivalence up to positive constant factors, respectively. For a DAG $G=(V,E)$ with node set $V=[p]=\{1,\ldots,p\}$ and directed edge set $E\subseteq V\times V$, write $k\to j$ if $(k,j)\in E$. A directed path from $a$ to $b$ is a sequence $\Gamma=(v_0\to\cdots\to v_r)$ of distinct nodes with $v_0=a$, $v_r=b$, and length $\ell(\Gamma)=r\ge1$. Let $\mathcal P(a,b)$ denote the set of these paths. An ordering is a permutation $\pi=(\pi_1,\ldots,\pi_p)$ of $V$, compatible with $G$ if $k\to j$ implies $\pi^{-1}(k)<\pi^{-1}(j)$. For $j\in V$, define $\pa(j)=\{k\in V:k\to j\}$ and $\an(j)=\{k\in V:\mathcal P(k,j)\ne\emptyset\}$. Let $d=\max_{j\in V}|\pa(j)|$ be the maximum indegree, and let $q$ denote a supplied indegree bound. A set $S\subseteq V$ is ancestral if $\pa(j)\subseteq S$ for every $j\in S$. For ancestral $S$, let $R(S)=V\setminus S$, $\Src(S)=\{j\in R(S):\pa(j)\subseteq S\}$, and $\Nsrc(S)=R(S)\setminus\Src(S)$ denote the remaining nodes, sources, and nonsources, respectively.

\subsection{Gaussian Linear DAG Models}
\label{SecLine}
\paragraph{Gaussian linear DAG model.}
The Gaussian linear DAG model is defined by the structural equations
\begin{equation}\label{eq:linear-sem}
X_j=\sum_{k\in\pa(j)}B_{kj}X_k+\varepsilon_j,
\qquad \varepsilon_j\overset{\mathrm{ind}}{\sim}N(0,\sigma_j^2),
\qquad j\in V,
\end{equation}
where $\sigma_j>0$ and $B_{kj}$ is the direct effect from $k$ to $j$. The nonzero entries of $B$ define the edges of the DAG $G$. The resulting vector $X\sim N(0,\Sigma)$ has covariance and precision matrices
\[
\Sigma=\Cov(X)=(I_p-B)^{-\top}\Sigma^\varepsilon(I_p-B)^{-1}\succ0,
\qquad \Omega=\Sigma^{-1},
\]
where $\Sigma^\varepsilon=\diag(\sigma_1^2,\ldots,\sigma_p^2)$ is the covariance matrix of the independent errors. The directed-path expansion of \eqref{eq:linear-sem} expresses each variable in terms of these errors \citep{sullivant2010trek}:
\begin{equation}\label{eq:error-path}
X_j=\varepsilon_j+\sum_{k\in\an(j)}B_{k\pathto j}\varepsilon_k,
\qquad
B_{k\pathto j}
=\sum_{\Gamma\in\mathcal P(k,j)}
\prod_{t=1}^{\ell(\Gamma)}B_{v_{t-1}v_t}.
\end{equation}
For each directed path $\Gamma\in\mathcal P(k,j)$, the product of the direct-effect coefficients along its edges gives the effect transmitted along $\Gamma$, and summing over these paths yields the total effect $B_{k\pathto j}$.

\paragraph{DAG identifiability.}
For $S\subseteq V$ and $j\notin S$, define the population ordering score of $j$ given $S$ as
$$
\mathcal S_j(S)=\E[(X_j-X_S^\top\theta_j(S))^2], \quad \theta_j(S)=\argmin_{\theta \in \mathbb{R}^{|S| }} \E[(X_j-X_S^\top\theta)^2].
$$
Under Gaussianity, the least-squares residual is independent of $X_S$, so $\mathcal S_j(S)=\Var(X_j\mid X_S)$. When the selected set $S$ is ancestral, the remaining sources $\Src(S)$ have all their parents in $S$ and can appear next in a compatible ordering. The following lemma gives a score separation under which minimizing the population score selects a remaining source.

\begin{lemma}[Forward Identifiability]\label{lem:identifiability}
Suppose every ancestral $S\subsetneq V$ satisfies
\[
\min_{k\in\Nsrc(S)}\mathcal S_k(S)
>\min_{j\in\Src(S)}\mathcal S_j(S),
\]
where a minimum over the empty set is $+\infty$. Then an ordering compatible with $G$ can be identified from the observational distribution.
\end{lemma}
Starting from $S=\emptyset$, successively adding score minimizers to $S$ keeps $S$ ancestral, so the selection order is compatible with $G$. The proof follows an inductive argument based on conditional variances, similar to those in \citet{chen2019causal} and \citet{park2020conditional}, so we omit the details.

For ancestral $S$ and $j\notin S$, conditioning on $X_S$ is equivalent to conditioning on $\varepsilon_S$, so \eqref{eq:error-path} and independence of the errors give the explicit form of the population ordering score
\begin{equation}\label{eq:population-score-path}
\mathcal S_j(S)=\sigma_j^2+
\sum_{a\in\an(j)\setminus S}B_{a\pathto j}^{\,2}\sigma_a^2.
\end{equation}
For a remaining source, every parent has been selected and the score is simply $\sigma_j^2$. For a remaining nonsource $j$, choose an unselected parent $a$ with no descendant among the other unselected parents of $j$. Then $B_{a\pathto j}=B_{aj}\ne0$, so at least one additional variance term is positive. Equal error variances therefore suffice for the score separation in Lemma~\ref{lem:identifiability} \citep{peters2014identifiability}. With heterogeneous errors, the comparison instead depends jointly on direct effects and error variances.

\section{One-Sided Score Control for Ordering Recovery}
\label{Sec:NewSparsity}

\subsection{Score Estimation Along Forward Search}
Implementing the ordering rule in Lemma~\ref{lem:identifiability} requires estimating the population scores $\mathcal S_j(S)$ at each selected ancestral set $S$. A natural approach uses the residual mean square from a regression of each remaining node on $X_S$ as a score estimator. The following proposition describes the corresponding population regression coefficients.

\begin{proposition}[Regression Coefficients at an Ancestral Set]
\label{prop:ancestral-regression}
In the Gaussian linear DAG model, let $S\subsetneq V$ be ancestral and $j\notin S$. The coefficient $\theta_{kj}(S)$ of $X_k$ in the population regression of $X_j$ on $X_S$ satisfies
\[
\theta_{kj}(S)=B_{k\pathto j}
-\sum_{a\in S\cap\an(j)}B_{ka}B_{a\pathto j}
=B_{kj}+\sum_{a\in\an(j)\setminus S}B_{ka}B_{a\pathto j},\qquad k\in S.
\]
\end{proposition}
The proof is given in Appendix~\ref{app:ancestral-regression-proof}. For a remaining source $j$, all parents lie in $S$, so $\theta_j(S)=B_{Sj}$ and the regression coincides with its structural equation. For a remaining nonsource, the coefficients also contain indirect effects through unselected nodes. Thus the maximum indegree $d$ does not bound the number of nonzero coefficients in the regression of a remaining nonsource on $X_S$.

\paragraph{Motivating example.}
In the layered DAG of Figure~\ref{graph:layered}(a), every node after the first layer has two parents, yet the regression of the designated remaining nonsource $X_{\rm ns}$ on $X_S$ includes indirect effects from all 64 variables in $S$. The regression of the designated remaining source $X_{\rm src}$ involves only its two parents.

\begin{figure}[htb]
\centering
\begin{tabular}{@{}c@{\hspace{0.12cm}}c@{}}
\begin{minipage}[b]{.273\linewidth}
\centering
\resizebox{\linewidth}{!}{%
\begin{tikzpicture}[x=0.72cm,y=0.078cm]
	\foreach \lev/\last in {0/63,1/31,2/15,3/7,4/3,5/1,6/0}{
		\foreach \i in {0,...,\last}{
			\pgfmathsetmacro{\yy}{pow(2,\lev)*\i+(pow(2,\lev)-1)/2}
			\coordinate (L\lev-\i) at (\lev,\yy);
		}
	}
	\fill[black!10,rounded corners=1pt] (-.27,-.8) rectangle (.27,63.8);
	\foreach \lev/\last in {1/31,2/15,3/7,4/3,5/1,6/0}{
		\pgfmathtruncatemacro{\prevlev}{\lev-1}
		\foreach \i in {0,...,\last}{
			\pgfmathtruncatemacro{\leftchild}{2*\i}
			\pgfmathtruncatemacro{\rightchild}{2*\i+1}
			\draw[black!50,line width=.22pt]
				(L\prevlev-\leftchild)--(L\lev-\i);
			\draw[black!50,line width=.22pt]
				(L\prevlev-\rightchild)--(L\lev-\i);
		}
	}
	\foreach \lev/\last in {0/63,1/31,2/15,3/7,4/3,5/1,6/0}{
		\foreach \i in {0,...,\last}{
			\fill[black!65] (L\lev-\i) circle[radius=.55pt];
		}
	}
	\foreach \i in {0,...,63}{
		\fill[blue!60!black] (L0-\i) circle[radius=.70pt];
	}
\draw[blue!70!black,line width=.65pt] (L0-0)--(L1-0);
\draw[blue!70!black,line width=.65pt] (L0-1)--(L1-0);
\fill[blue!70!black] (L1-0) circle[radius=1.65pt];
	\fill[red!75!black] (L6-0) circle[radius=1.65pt];
	\node[font=\footnotesize\bfseries\boldmath,text=blue!70!black,anchor=south west,
		xshift=1pt,yshift=1pt] at (L1-0) {$X_{\rm src}$};
		\node[font=\footnotesize\bfseries\boldmath,text=red!75!black,anchor=west,
			xshift=2pt] at (L6-0) {$X_{\rm ns}$};
	\node[font=\footnotesize\bfseries\boldmath,text=black,rotate=90] at (-.45,31.5)
		{ancestral set $S$};
	\draw[-latex,black!65,line width=.45pt] (0,67)--(6,67)
		node[midway,above,font=\footnotesize] {ordering};
	\foreach \x/\lab in {0/64,1/32,2/16,3/8,4/4,5/2,6/1}{
		\node[font=\scriptsize] at (\x,-4.5) {\lab};
	}
	\node[font=\scriptsize] at (3,-9) {layer width};
\end{tikzpicture}
}%
\par
{\footnotesize (a) Layered DAG\par}
\end{minipage} &
\begin{minipage}[b]{.651\linewidth}
\centering
\includegraphics[width=\linewidth]
{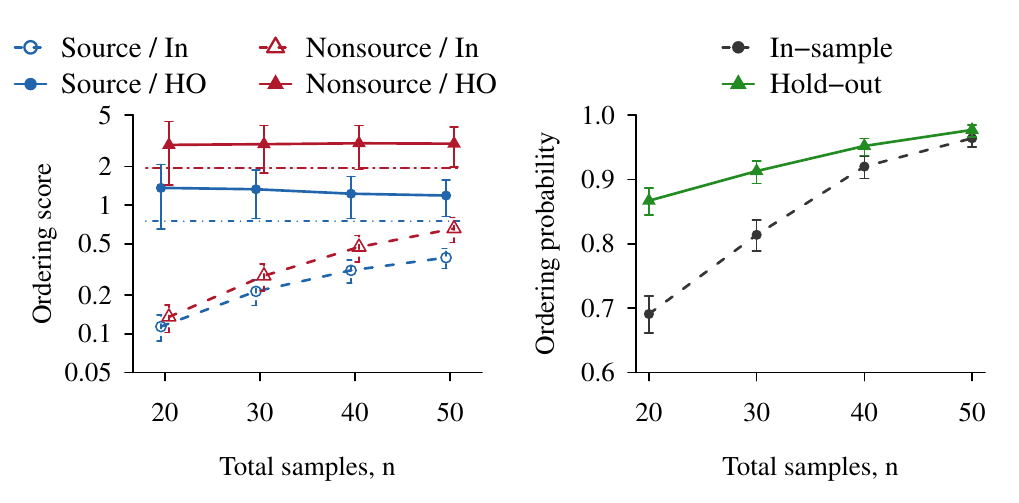}
\par
\noindent\begin{minipage}[t]{.5\linewidth}
\centering{\footnotesize (b) Ordering scores\par}
\end{minipage}%
\begin{minipage}[t]{.5\linewidth}
\centering{\footnotesize (c) Ordering probability\par}
\end{minipage}
\end{minipage}
\end{tabular}
\caption{Panel (a) shows a 127-node layered DAG with the ancestral set $S$ shaded, the designated remaining source $X_{\rm src}$ in blue, and the designated remaining nonsource $X_{\rm ns}$ in red. Panel (b) shows mean Lasso ordering scores for in-sample (In) and hold-out (HO) scoring using $X_S$ on a log scale. Error bars indicate one standard deviation across repetitions, computed on the original score scale. Horizontal lines mark population scores of 0.75 for $X_{\rm src}$ and 1.93 for $X_{\rm ns}$. Panel (c) shows the fraction of repetitions with $X_{\rm src}$ scoring below $X_{\rm ns}$, with 95\% Wilson intervals. Appendix~\ref{app:score-diagnostic} gives the settings.}
\label{graph:layered}
\end{figure}

Minimax lower bounds for residual variance estimation in high-dimensional Gaussian regression can depend on the number of nonzero regression coefficients \citep{verzelen2018adaptive}. As Proposition~\ref{prop:ancestral-regression} and the preceding example show, indirect effects can make this number much larger than the maximum indegree $d$ for remaining nonsources, making accurate estimation of every score difficult even when $d$ is small. However, preserving the population score comparison in Lemma~\ref{lem:identifiability} suffices for ordering recovery. In other words, it suffices to bound score overestimation for remaining sources and score underestimation for remaining nonsources so that the estimated scores preserve the comparison.

\subsection{Hold-Out Scoring}
Score underestimation can be controlled through hold-out scoring because fitting the regression on an independent sample prevents overfitting to the evaluation observations. The following lemma provides a bound on score underestimation that does not depend on the true regression coefficients or the accuracy of their estimates.

\begin{lemma}[Hold-Out Lower Tail]\label{lem:holdout-lower-tail}
Let $\mathcal F$ be a sigma-field. For integers $n\ge1$ and $s\ge0$, suppose $\mathbf X\in\R^{n\times s}$ and both vectors $\theta,\widetilde\theta\in\R^s$ are $\mathcal F$-measurable with finite entries. Suppose $Y=\mathbf X\theta+\xi$, where $\xi\mid\mathcal F\sim N(0,vI_n)$ for a fixed $v>0$. Then for $\widetilde{\mathcal S}=\|Y-\mathbf X\widetilde\theta\|_2^2/n$ and every $0<t<v$,
$$
\P\{\widetilde{\mathcal S}<v-t\mid\mathcal F\}
\le\exp\{-nt^2/(4v^2)\}\qquad\text{almost surely}.
$$
\end{lemma}
The proof is given in Appendix~\ref{app:holdout-proof}. For the regression on $X_S$ at a fixed ancestral set $S$, apply the lemma with $v=\mathcal S_j(S)$ and $\theta=\theta_j(S)$, conditioning on the fitted coefficients and held-out values of $X_S$. With $\mathbf X$ denoting the design matrix of the $n$ held-out observations, the conditional expected score is $\E[\widetilde{\mathcal S}\mid\mathcal F]=\mathcal S_j(S)+\|\mathbf X(\theta_j(S)-\widetilde\theta)\|^2/n\ge\mathcal S_j(S)$. Thus coefficient estimation error contributes a nonnegative term, providing intuition for the lower-tail bound. The lemma controls score underestimation for remaining nonsources regardless of the true coefficients, their estimates, or the size of $S$. For remaining sources, score overestimation is controlled through the accuracy of the fitted regression. Proposition~\ref{prop:ancestral-regression} gives $\theta_j(S)=B_{Sj}$, so the population regression involves at most $d$ nonzero coefficients. Estimation bounds for penalized regression can therefore provide the complementary control of score overestimation with sample requirements depending on $d$.

Figure~\ref{graph:layered}(b) and (c) illustrate the distinction between accurate score estimation and preservation of the score comparison. In-sample scoring uses all $n$ observations for both fitting and evaluation, while hold-out scoring uses one half for fitting and the other for evaluation. In-sample scores substantially underestimate the population scores and reduce their separation in panel (b), allowing the remaining nonsource to be selected first. Hold-out scores exceed the population scores on average, yet panel (c) shows a higher probability of placing the remaining source below the remaining nonsource.

\section{HOST Algorithm and Guarantees}
\label{SecAlg}
HOST builds the graph one node at a time. It selects the next node by hold-out scoring, recovers its parents among the selected nodes, and retains a local ordinary least squares (OLS) fit and its residual for later iterations. The parent-recovery step uses these residuals to separate direct effects from indirect effects, with the ordering fits guiding a bounded correction before thresholding.

\paragraph{Initialization.}
Let $\mathbf X$ be the data matrix from $n$ independent observations. Split the observations independently of their values into $I_1\cup I_2=[n]$, with sizes $n_\ell=|I_\ell|$ differing by at most one. For $\ell\in\{1,2\}$, write $\bar\ell$ for the other fold index. Use the same positive ordering penalty $\lambda_{\mathrm{ord}}$ on both folds, positive parent thresholds $\tau_j$, and a fixed clipping parameter $0\le\delta<1$. Choose $\widehat\pi_1=\argmin_{j\in V}\widehat\Sigma_{jj}$, where $\widehat\Sigma_{jj}=\|\mathbf X_j\|^2/n$, and initialize the estimated parent set $\widehat T_{\widehat\pi_1}=\emptyset$ and the residual $\widehat r_{\widehat\pi_1}=\mathbf X_{\widehat\pi_1}$.

\paragraph{Ordering by hold-out scoring.}
At iteration $r\ge2$, the selected nodes form $\widehat S_r=\{\widehat\pi_1,\ldots,\widehat\pi_{r-1}\}$. For every remaining node $j$, fit a Lasso regression on each half of the sample,
\begin{equation}\label{eq:algordering}
\widehat\theta_j^{(\ell)}(r)\in\argmin_\theta
\Bigl\{\frac{\|\mathbf X_{I_\ell j}-\mathbf X_{I_\ell\widehat S_r}\theta\|^2}{2n_\ell}
+\lambda_{\mathrm{ord}}\|\theta\|_1\Bigr\},\qquad \ell=1,2.
\end{equation}
Combine the two opposite-fold residual sums of squares to obtain the hold-out ordering score,
\begin{equation}\label{eq:algscores}
\widehat{\mathcal S}_j(r)=\frac1n\sum_{\ell=1}^2
\|\mathbf X_{I_{\bar\ell}j}-\mathbf X_{I_{\bar\ell}\widehat S_r}
\widehat\theta_j^{(\ell)}(r)\|^2.
\end{equation}
The remaining node with the smallest score becomes $\widehat\pi_r$, with ties resolved by node label. For $j=\widehat\pi_r$, let $\bar\theta_{kj}=\{\widehat\theta_{kj}^{(1)}(r)+\widehat\theta_{kj}^{(2)}(r)\}/2$ be the averaged coefficient of $X_k$ for each earlier node $k$. Define the updated ordering coefficients by refitting each nonzero coordinate as
\begin{equation}\label{eq:coordinate-refit}
\check\theta_{kj}=\argmin\nolimits_{b\in\mathbb R}\sum\nolimits_{i=1}^n
\Bigl(\mathbf X_{ij}
-\sum\nolimits_{a\in\widehat S_r\setminus\{k\}}\mathbf X_{ia}\bar\theta_{aj}
-\mathbf X_{ik}b\Bigr)^2,
\qquad \bar\theta_{kj}\ne0,
\end{equation}
and set $\check\theta_{kj}=0$ for $\bar\theta_{kj}=0$. This update aims to reduce Lasso shrinkage and guides correction of parent estimates from local fits below, without requiring the ordering Lasso to recover the parent set.

\paragraph{Estimating total effects.}
Each $k \in \widehat S_r$ has a previously estimated parent set $\widehat T_k$. The direct-effect estimates from $\widehat T_k$ to $k$ and the residual from its retained local OLS fit are reused from earlier iterations, and are given by
\begin{equation}\label{eq:local-ols}
\widehat B^{\mathrm{LS}}_{\widehat T_k,k}
=\mathbf X_{\widehat T_k}^{\dagger}\mathbf X_k,\quad
\widehat B^{\mathrm{LS}}_{V\setminus\widehat T_k,k}=0,\quad
\widehat r_k=\mathbf X_k-\mathbf X\widehat B^{\mathrm{LS}}_{\cdot k},
\end{equation}
where $\dagger$ denotes the Moore--Penrose inverse. When $\widehat T_k=\pa(k)$, the population residual is $\varepsilon_k$, and \eqref{eq:error-path} gives $\Cov(\varepsilon_k,X_j)/\sigma_k^2=B_{k\pathto j}$. We therefore use $\widehat B_{k\pathto j}=\widehat r_k^\top\mathbf X_j/\|\widehat r_k\|^2$, defined as zero when $\widehat r_k=0$, to estimate the total effect of $X_k$ on $X_j$.

\paragraph{Structural thresholding.}
For $j=\widehat\pi_r$, let $\widehat T_j^{(m)}$ and $\widehat Q_j^{(m)}$ contain the selected parents and nodes retained for subtraction after $m$ steps, both initially empty. At step $m=1,\ldots,r-1$, set $k=\widehat\pi_{r-m}$. Proposition~\ref{prop:ancestral-regression} gives the direct-effect identity and its sample reconstruction:
\begin{equation}\label{eq:structural-pilot}
B_{kj}=B_{k\pathto j}-\sum\nolimits_{a\in\an(j)}B_{ka}B_{a\pathto j},
\qquad
\widehat B^{\mathrm{res}}_{kj}=\widehat B_{k\pathto j}
-\sum\nolimits_{a\in\widehat Q_j^{(m-1)}}\widehat B^{\mathrm{LS}}_{ka}\widehat B_{a\pathto j}.
\end{equation}
With $\clip_t(x)=\max\{-t,\min(x,t)\}$, apply a bounded correction using $\check\theta_{kj}$ to obtain a corrected direct-effect estimate. Select $k$ as a parent of $j$ when the magnitude of this estimate exceeds $\tau_j$:
\begin{equation}\label{eq:corrected-statistic}
\widetilde B_{kj}=\widehat B^{\mathrm{res}}_{kj}
+\clip_{\delta\tau_j}\!\left(\check\theta_{kj}-\widehat B^{\mathrm{res}}_{kj}\right),
\qquad
\widehat T_j^{(m)}=\widehat T_j^{(m-1)}\cup
\{u\in\{k\}:|\widetilde B_{kj}|>\tau_j\}.
\end{equation}
The correction is bounded by $\delta\tau_j$, so parent recovery does not require $\check\theta_{kj}$ to be accurate. Retain $k$ if it is selected as a parent of $j$ or belongs to the estimated parent set of a previously retained node:
\begin{equation}\label{eq:reverse-set}
\widehat Q_j^{(m)}=\widehat Q_j^{(m-1)}\cup
\Bigl(\{k\}\cap\Bigl(\widehat T_j^{(m)}
\cup\bigcup\nolimits_{a\in\widehat Q_j^{(m-1)}}\widehat T_a\Bigr)\Bigr).
\end{equation}
After $r-1$ steps, return $\widehat T_j=\widehat T_j^{(r-1)}$ and use $\widehat T_j$ to compute the local OLS fit $\widehat B^{\mathrm{LS}}_{\cdot j}$ and residual $\widehat r_j$ in \eqref{eq:local-ols} for subsequent iterations.

\begin{algorithm}[htbp]
\caption{\textbf{HOST} Hold-Out Ordering and Structural Thresholding}\label{alg1}
\SetAlgoCaptionSeparator{. }
\SetKwInOut{Input}{Input}\SetKwInOut{Output}{Output}\DontPrintSemicolon
\renewcommand{\Input}[1]{\textbf{Input} #1\;}
\renewcommand{\Output}[1]{\textbf{Output} #1\;}
\Input{Data $\mathbf X$, penalty $\lambda_{\mathrm{ord}}>0$, parent thresholds $\tau_j>0$, and $\delta\in[0,1)$}
\Output{Estimated DAG $\widehat G$}
Split $[n]$ into balanced folds $I_1,I_2$ independently of $\mathbf X$\;
Set $\widehat\pi_1=\argmin_{j\in V}\widehat\Sigma_{jj}$, $\widehat T_{\widehat\pi_1}=\emptyset$, and residual $\widehat r_{\widehat\pi_1}=\mathbf X_{\widehat\pi_1}$\;
\For{$r=2,\ldots,p$ with $\widehat S_r=\{\widehat\pi_1,\ldots,\widehat\pi_{r-1}\}$}{
Fit foldwise Lasso regressions and compute hold-out scores for $j\notin\widehat S_r$ by \eqref{eq:algordering}--\eqref{eq:algscores}\;
Select $j=\widehat\pi_r=\argmin_{a\notin\widehat S_r}\widehat{\mathcal S}_a(r)$ and compute $\check\theta_j$ by \eqref{eq:coordinate-refit}\;
Set $\widehat T_j^{(0)}=\widehat Q_j^{(0)}=\emptyset$ and compute $\widehat B_{k\pathto j}=\widehat r_k^\top\mathbf X_j/\|\widehat r_k\|^2$ for $k\in\widehat S_r$\;
\For{$m=1,\ldots,r-1$}{
Set $k=\widehat\pi_{r-m}$ and subtract indirect effects to obtain $\widehat B^{\mathrm{res}}_{kj}$ by \eqref{eq:structural-pilot}\;
Correct and threshold by \eqref{eq:corrected-statistic} to update $\widehat T_j^{(m)}$, then update $\widehat Q_j^{(m)}$ by \eqref{eq:reverse-set}\;
}
Set $\widehat T_j=\widehat T_j^{(r-1)}$ and store its OLS coefficients and residual by \eqref{eq:local-ols}\;
}
Return $\widehat G=(V,\{k\to j:j\in V,\ k\in\widehat T_j\})$\;
\end{algorithm}

\subsection{Theoretical Guarantees}
\label{SecMain}

Throughout this section, we assume $p\ge2$ and even $n\ge4$ and use a data-independent split into two equal halves. Define $B_{\min}=\min_{k\to j}|B_{kj}|$, with $B_{\min}=+\infty$ for an empty graph. Define $\Sigma_{\max}=\max_j\Sigma_{jj}$, $\sigma_{\max}=\max_j\sigma_j$, and $\Omega_{\max}=\max_k\Omega_{kk}$.

\begin{assumption}[Forward Identifiability Margin]
\label{assu:taumin}
There is a margin $\tau_\star>0$ such that, for every ancestral $S\subsetneq V$ with at least one remaining nonsource,
\[
\min_{k\in\Nsrc(S)}\mathcal S_k(S)
-\min_{j\in\Src(S)}\mathcal S_j(S)\ge\tau_\star.
\]
\end{assumption}
For fixed $p$, the strict separation in Lemma~\ref{lem:identifiability} yields a positive common margin because there are finitely many ancestral sets. With common error variance $\sigma^2$ and $E\ne\emptyset$, one can take $\tau_\star=\sigma^2B_{\min}^2$.

\begin{theorem}[Ordering Recovery]
\label{thm:ordering-recovery}
Suppose the data follow the Gaussian linear DAG model in Section~\ref{SecLine} and Assumption~\ref{assu:taumin} holds. Choose $\lambda_{\mathrm{ord}}\asymp \sigma_{\max}\sqrt{\Sigma_{\max}\log p/n}$, with sufficiently large universal constant factors. If
\begin{equation}
\label{eq:ordering-rate}
n\gtrsim\left[
\frac{d\Sigma_{\max}}{\lambda_{\min}(\Sigma)}
\left(1+\frac{\sigma_{\max}^2}{\tau_\star}\right)
+\left(1+\frac{\Sigma_{\max}}{\tau_\star}\right)^2
\right]\log p,
\end{equation}
then $\P\{\widehat\pi\text{ is compatible with }G\}\ge1-p^{-2}$.
\end{theorem}

The proof is given in Appendix~\ref{app:ordering-proof}. It combines score estimation bounds for remaining sources with hold-out lower-tail bounds for remaining nonsources. The argument uses conditional bounds for selected sets to account for dependence induced by adaptive node selection across iterations.

\begin{theorem}[Parent Recovery]
\label{thm:parent-recovery}
Under the Gaussian linear DAG model, fix $0\le\delta<1$ and use $\tau_j\asymp\sqrt{\Sigma_{\max}\Omega_{\max}\log p/n}$ in Algorithm~\ref{alg1}, with a sufficiently large constant depending only on $\delta$. If
\begin{equation}\label{eq:parent-rate}
n\gtrsim d+\log p,\qquad
B_{\min}\gtrsim\sqrt{{\Sigma_{\max}\Omega_{\max}\log p}/{n}},
\end{equation}
then, for any ordering $\widehat\pi$ computed from the same observations,
\[
\P\{\widehat G\ne G\}
\le\P\{\widehat\pi\text{ is incompatible with }G\}+p^{-2}.
\]
\end{theorem}

The proof is given in Appendix~\ref{app:parent-proof}. It uniformly controls errors in the reconstructed direct effects over compatible orderings. An induction argument then shows that thresholding recovers all parents, even when the ordering is learned from the same data, without requiring accurate ordering fits. Ordering guarantees concern the exact minimizers in \eqref{eq:algordering} at the fixed penalty specified above.

\begin{theorem}[Exact DAG Recovery]
\label{thm:dag-recovery}
Suppose Assumption~\ref{assu:taumin} holds in the Gaussian linear DAG model. Choose the fixed ordering penalty and parent thresholds in Theorems~\ref{thm:ordering-recovery} and~\ref{thm:parent-recovery}. If
\begin{equation}\label{eq:dag-rate}
n\gtrsim\left[
\frac{d\Sigma_{\max}}{\lambda_{\min}(\Sigma)}
\left(1+\frac{\sigma_{\max}^2}{\tau_\star}\right)
+\left(1+\frac{\Sigma_{\max}}{\tau_\star}\right)^2
+\frac{\Sigma_{\max}\Omega_{\max}}{B_{\min}^2}
\right]\log p,
\end{equation}
then $\P\{\widehat G=G\}\ge1-p^{-1}$.
\end{theorem}

The proof in Appendix~\ref{app:dag-recovery-proof} combines the ordering and parent bounds. This high-dimensional bound gives recovery probability tending to one as $p=p_n\to\infty$ under \eqref{eq:dag-rate}.

For computational complexity, we use cyclic coordinate descent for Lasso and singular value decomposition (SVD) for minimum-norm local OLS fits. We report complexity for dense implementations, suppressing solver iteration counts and numerical tolerance factors.

\begin{proposition}[Computational Complexity]\label{prop:computational-complexity}
Under the computational setup above, Algorithm~\ref{alg1} has arithmetic complexity $O(np^3+p^4)$ for a fixed ordering penalty and fixed parent thresholds.
\end{proposition}
Appendix~\ref{app:computation-proof} provides the proof. The complexity bound concerns the numerical implementation, whereas the statistical guarantees use exact Lasso minimizers.

\subsection{Comparison with Existing Forward-Search Methods}\label{sec:statistical-comparison}
For equal-error-variance models, let $\sigma^2$ be the common error variance. Under $M^{-1}I\preceq\Sigma\preceq MI$ for a constant $M>1$, \citet{gao2022optimal} establish the sample complexity lower bound $n\gtrsim\max\{B_{\min}^{-2}\log p,\,(M^2-1)^{-1}d\log(p/d)\}$ for exact DAG recovery when $1\le d\le p/2$. Under their convention of fixed $M$, $B_{\min}$, and $\sigma^2$, their subset-search algorithm attains minimax-optimal $d\log(p/d)$ sample complexity. For HOST, the covariance bounds and $\tau_\star=\sigma^2B_{\min}^2$ reduce \eqref{eq:dag-rate} to the sufficient condition
\[
n\gtrsim\bigl[dM^2(1+B_{\min}^{-2})+(1+M\sigma^{-2}B_{\min}^{-2})^2+M^2B_{\min}^{-2}\bigr]\log p.
\]
Under the same convention, the sample complexity of HOST is within a factor of order $\log d$ of theirs for $2\le d\le p/2$. \citet{gao2022optimal} require a supplied indegree bound $q$ and $p^{O(q)}$ subset search, whereas HOST runs in polynomial time.

For equal-error-variance models without this eigenvalue restriction, \citet{daskalakis2025condition} establish a sample complexity lower bound $n\gtrsim\max\{B_{\min}^{-4},TB_{\min}^{-2}\}\log p$, where $T=\sigma^2\Omega_{\max}=1+\max_i\sum_{j\colon i\to j}B_{ij}^2$. For HOST, using the parameter relationships $\tau_\star=\sigma^2B_{\min}^2$, $\sigma^2/\lambda_{\min}(\Sigma)\le(d+1)T$, and $V=\Sigma_{\max}/\sigma^2\ge1+B_{\min}^2$ for a nonempty DAG reduces \eqref{eq:dag-rate} to the sufficient condition
\begin{equation}\label{eq:host-lower-comparison}
n\gtrsim\bigl[d^2VT(1+B_{\min}^{-2})+V^2B_{\min}^{-4}\bigr]\log p.
\end{equation}
Their Algorithm~1 has sample complexity $n\gtrsim\max\{B_{\min}^{-4},TB_{\min}^{-2}\}\,d\log(p/d)$ and is based on subset search. For bounded $V$, this sample complexity is within a factor $d$ of the lower bound and has milder $d$ dependence than HOST, but Algorithm~1 requires a supplied indegree bound $q$ and $p^{O(q)}$ computation. Their Algorithms~2 and~3 provide polynomial-time alternatives with sample complexity $n\gtrsim Vd^4T^3B_{\min}^{-4}\log p$. For bounded $V$, HOST improves the dependence on $d$ and $T$ from $d^4T^3$ to $d^2T$, although the dependence on $V$ differs.

\section{Simulation Studies}\label{sec:simulation}
We evaluate HOST across sample sizes, dimensions, and maximum indegrees, comparing its recovery accuracy and runtime with existing methods. All experiments use Gaussian linear DAGs. In this section, each node has at most $d$ randomly selected parents. Nonzero direct effects have magnitudes uniform on $[0.4,0.6]$ and symmetric signs, and independent Gaussian errors have variance $0.75$. Appendix~\ref{app:simulation} details the experimental settings and includes additional graph families (\ref{app:additional-graphs}), ordering comparisons (\ref{app:ordering-bound}), and parent-recovery comparisons (\ref{app:oracle-parent}).

Figure~\ref{fig:host-scaling} shows improving graph recovery with increasing sample size, with qualitative trends consistent with Theorem~\ref{thm:dag-recovery}. The experiment covers $p$ = 30, 50, 100, 200, $d$ = 2, 3, 4, 5, and $n$ = 128, 256, 512, 1024, 2048, with 200 repetitions per setting. The graph recovery rate is the fraction of exactly recovered graphs. For each $(d,p)$, the rate tends toward one as $n$ grows. Panels (a)--(d) vary maximum indegree $d$ at fixed dimension $p$ and generally show higher graph recovery rates for smaller $d$ at a given sample size. Panels (e)--(h) vary dimension $p$ at fixed maximum indegree $d$ and show that larger graphs generally require more observations to attain a comparable graph recovery rate.

\begin{figure}[htb]\centering
\captionsetup[subfloat]{captionskip=2pt,farskip=0pt,nearskip=0pt}
\subfloat[$p$ = 30]{\includegraphics[width=.272727\linewidth]{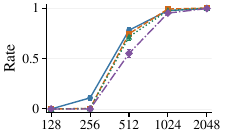}\label{fig:host-scaling-p30}}\hfill
\subfloat[$p$ = 50]{\includegraphics[width=.227273\linewidth]{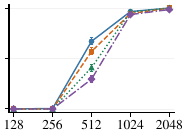}\label{fig:host-scaling-p50}}\hfill
\subfloat[$p$ = 100]{\includegraphics[width=.227273\linewidth]{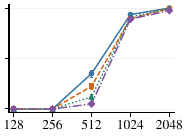}\label{fig:host-scaling-p100}}\hfill
\subfloat[$p$ = 200]{\includegraphics[width=.227273\linewidth]{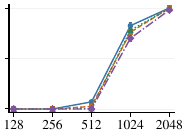}\label{fig:host-scaling-p200}}\par
\subfloat[$d$ = 2]{\includegraphics[width=.272727\linewidth]{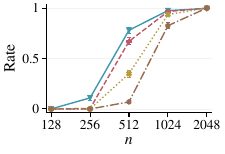}\label{fig:host-scaling-d2}}\hfill
\subfloat[$d$ = 3]{\includegraphics[width=.227273\linewidth]{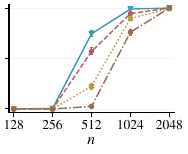}\label{fig:host-scaling-d3}}\hfill
\subfloat[$d$ = 4]{\includegraphics[width=.227273\linewidth]{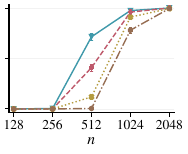}\label{fig:host-scaling-d4}}\hfill
\subfloat[$d$ = 5]{\includegraphics[width=.227273\linewidth]{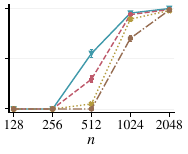}\label{fig:host-scaling-d5}}\par
\includegraphics[width=\linewidth]{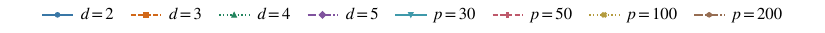}\par
\caption{Graph recovery rate of HOST with $\pm1$ standard error over 200 repetitions. Panels (a)--(d) overlay curves for different $d$ at fixed $p$, and (e)--(h) for different $p$ at fixed $d$.}
\label{fig:host-scaling}\end{figure}

Figure~\ref{fig:all} compares HOST with fixed tuning and HOST-CV/EBIC with data-selected tuning (Appendix~\ref{app:ebic-proof}) against HLSM \citep{park2021learning2}, TD-BSS and TD-BIC \citep{gao2022optimal}, TD-Lasso \citep{chen2019causal}, Daskalakis Algs.~1 and~2 \citep{daskalakis2025condition}, and GOLEM-EV \citep{ng2020role}. Graph recovery is evaluated using the directed F1 score at $p$ = 50, 200, $d$ = 3, and $n$ = 100, 200, 300, 400, 500, with 50 paired repetitions. Full-fit runtime, including parameter selection where applicable, is evaluated at $n$ = 100, $p$ = 10, 30, 50, 70, 100, and $d$ = 1, 2, 3, 4, 5, with 20 paired repetitions. Daskalakis Alg.~1 is evaluated only for $p\le50$, and GOLEM-EV is excluded from timing. All recovery metrics and the full runtime comparison are reported in Appendices~\ref{app:simulation-tuning} and~\ref{app:timing-protocol}, respectively.

\begin{figure}[htb]\centering
\captionsetup[subfloat]{captionskip=2pt,farskip=0pt,nearskip=0pt}
\subfloat[Directed F1 versus $n$ at $p$ = 50]{\includegraphics[width=.398182\linewidth]{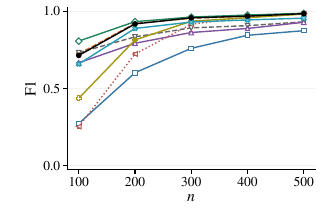}\label{fig:main-f1-p50}}\hspace{.010909\linewidth}%
\subfloat[Directed F1 versus $n$ at $p$ = 200]{\includegraphics[width=.358182\linewidth]{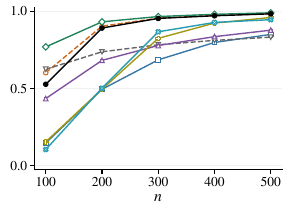}\label{fig:main-f1-p200}}\par
\subfloat[Runtime versus $d$ at $p$ = 100]{\includegraphics[width=.398182\linewidth]{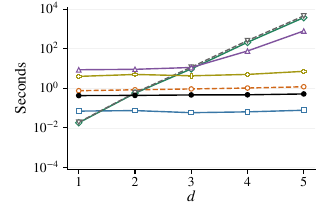}\label{fig:time-p100}}\hspace{.010909\linewidth}%
\subfloat[Runtime versus $p$ at $d$ = 5]{\includegraphics[width=.358182\linewidth]{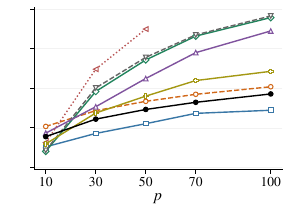}\label{fig:time-d5}}\par
\includegraphics[width=\linewidth]{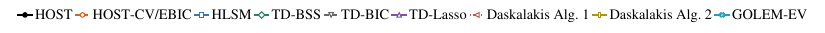}\par
\caption{Recovery accuracy and runtime scaling. Recovery uses $d$ = 3 and shows means $\pm1$ standard error over 50 paired repetitions. Runtime uses $n$ = 100 and reports mean full-fit time in seconds (log scale), including tuning, over 20 paired repetitions.}
\label{fig:all}\end{figure}

Figure~\ref{fig:all} shows the main empirical advantage of \textsc{HOST}: at sufficiently large sample sizes, it achieves recovery accuracy close to TD-BSS, while its runtime grows much more slowly with $d$ and it requires no supplied indegree bound. In panels (a)--(b), \textsc{HOST}-CV/EBIC improves F1 over fixed-tuning \textsc{HOST} at $n=100$ for both values of $p$. As the sample size increases, both \textsc{HOST} variants approach the F1 of TD-BSS. At $p=200$, $d=3$, and $n=500$, \textsc{HOST} and \textsc{HOST}-CV/EBIC attain directed F1 $0.984$ and $0.985$, respectively, compared with $0.991$ for TD-BSS. Panel (c) shows that TD-BSS is fast for small $d$ but becomes much slower as $d$ increases. At $p=100$, increasing $d$ from $1$ to $5$ raises its runtime from $0.018$ to $3716$ seconds, whereas the runtime of \textsc{HOST} changes only from $0.43$ to $0.52$ seconds and that of \textsc{HOST}-CV/EBIC changes from $0.76$ to $1.20$ seconds. Thus, at $d=5$, \textsc{HOST} is more than $7{,}000$ times faster than TD-BSS in this setting. Panel (d) shows a similar runtime advantage as $p$ increases at $d=5$. HLSM is consistently faster than \textsc{HOST}, but its F1 is consistently lower.

\section{Conclusion}
We prove a new theoretical principle underlying \textsc{HOST}: accurate estimation of every candidate score is unnecessary, and it suffices to control only errors that can cause an incorrect next-node selection. Building on this result, \textsc{HOST} achieves exact Gaussian DAG recovery with sample complexity of order $d\log p$ and polynomial computation, without a supplied indegree bound or subset search. Its parent-recovery analysis also avoids exact Lasso support recovery and restrictive conditions such as irrepresentability.

Beyond \textsc{HOST}, our analysis isolates a broader principle for fit-then-score structure learning: when correct selection depends only on one-sided score errors, held-out evaluation can remove the need for uniformly accurate estimation and prevent overfitting-induced premature selection. Extending this principle to nonlinear and non-Gaussian DAG models and developing adaptive tuning with recovery guarantees are promising directions for future work.

\subsection*{AI use statement}
We used OpenAI Codex to help draft sections of the manuscript, translate author-provided content into English, improve readability, identify relevant literature, provide critical ingredients for proving the authors' mathematical claims and assist in writing the proofs, implement methods by writing and debugging method, simulation, and analysis code, and clean and reformat datasets. We did not use generative AI tools to generate synthetic data sets, formulate mathematical claims, develop theoretical models or conceptual frameworks, propose or refine hypotheses, design or provide feedback on research methodology or experiments, or interpret results, and qualitative and thematic data analysis is not applicable to this work. The authors manually checked all mathematical claims and proofs, code, experimental results, and citations, and take full responsibility for the final content.

\bibliography{references}
\appendix
\section{Proofs for Regression Coefficients and Hold-Out Bounds}
\label{app:population-proofs}

\subsection{Proof of Proposition~\ref{prop:ancestral-regression}}
\label{app:ancestral-regression-proof}
\begin{proof}
If $S=\emptyset$, there is no coefficient to identify. Suppose $S\ne\emptyset$. Since $S$ is ancestral, every parent of a node in $S$ also lies in $S$, so the structural equations restricted to $S$ give
\[
\varepsilon_S=(I_{|S|}-B_{SS}^\top)X_S.
\]
Acyclicity makes $I_{|S|}-B_{SS}^\top$ invertible. Thus $X_S$ and $\varepsilon_S$ determine each other and are independent of $\varepsilon_{R(S)}$. Split the error expansion in \eqref{eq:error-path} as
\[
X_j=\sum_{a\in S}B_{a\pathto j}\varepsilon_a
+\varepsilon_j+\sum_{a\in R(S)\setminus\{j\}}B_{a\pathto j}\varepsilon_a.
\]
Taking conditional expectations given $X_S$ removes the errors indexed by $R(S)$. Then substituting $\varepsilon_a=X_a-\sum_{k\in S}B_{ka}X_k$ for $a\in S$ gives
\begin{equation}\label{pop:conditional-mean}
\begin{aligned}[b]
\E(X_j\mid X_S)
&=\sum_{a\in S}B_{a\pathto j}\varepsilon_a\\
&=\sum_{a\in S}B_{a\pathto j}X_a
-\sum_{a\in S}\sum_{k\in S}B_{ka}B_{a\pathto j}X_k\\
&=\sum_{k\in S}\left(B_{k\pathto j}
-\sum_{a\in S}B_{ka}B_{a\pathto j}\right)X_k.
\end{aligned}
\end{equation}
Since $\Sigma_{SS}\succ0$, the linear conditional mean in \eqref{pop:conditional-mean} identifies the unique population regression coefficient $\theta_j(S)$. Comparing coefficients and using $B_{a\pathto j}=0$ for $a\notin\an(j)$ yields
\begin{equation}\label{pop:regression-coefficient}
\theta_{kj}(S)=B_{k\pathto j}-\sum_{a\in S\cap\an(j)}B_{ka}B_{a\pathto j},
\qquad k\in S.
\end{equation}

For $k\ne j$, partitioning $\mathcal P(k,j)$ by the first edge and using the path-sum definition in \eqref{eq:error-path} gives
\begin{equation}\label{pop:first-edge-expansion}
\begin{aligned}[b]
B_{k\pathto j}
&=B_{kj}+\sum_{a\in\an(j)}B_{ka}
\left(\sum_{\Gamma\in\mathcal P(a,j)}\prod_{t=1}^{\ell(\Gamma)}B_{v_{t-1}v_t}\right)\\
&=B_{kj}+\sum_{a\in\an(j)}B_{ka}B_{a\pathto j}.
\end{aligned}
\end{equation}
Substituting \eqref{pop:first-edge-expansion} into \eqref{pop:regression-coefficient} cancels the terms indexed by $S\cap\an(j)$, yielding
\[
\theta_{kj}(S)=B_{kj}+\sum_{a\in\an(j)\setminus S}B_{ka}B_{a\pathto j}.
\]
For $k\ne j$ with $j\notin\an(k)$, take $S=\an(j)\cup\an(k)\cup\{k\}$, which is ancestral and excludes $j$. Since $\an(j)\subseteq S$, the two expressions for $\theta_{kj}(S)$ give the population identity in \eqref{eq:structural-pilot}. If $j\in\an(k)$, every term in that identity is zero by acyclicity.
\end{proof}

\subsection{Proof of Lemma~\ref{lem:holdout-lower-tail}}
\label{app:holdout-proof}
\begin{proof}
The conditional Gaussian model and $\mathcal F$-measurability of $\mathbf X,\theta,\widetilde\theta$ give $Y-\mathbf X\widetilde\theta\mid\mathcal F\sim N(\mathbf X(\theta-\widetilde\theta),vI_n)$. Since $n\widetilde{\mathcal S}=\|Y-\mathbf X\widetilde\theta\|_2^2$, Gaussian integration gives, for $u>0$,
\begin{equation}\label{holdout:laplace-bound}
\E[e^{-nu\widetilde{\mathcal S}}\mid\mathcal F]=(1+2vu)^{-n/2}\exp\left\{-\frac{u\|\mathbf X(\theta-\widetilde\theta)\|_2^2}{1+2vu}\right\}\le(1+2vu)^{-n/2}.
\end{equation}
Applying the Markov inequality conditionally on $\mathcal F$ and using \eqref{holdout:laplace-bound} gives
\begin{equation}\label{holdout:chernoff-bound}
\begin{aligned}[b]
\P\{\widetilde{\mathcal S}<v-t\mid\mathcal F\}
&=\P\{e^{-nu\widetilde{\mathcal S}}>e^{-nu(v-t)}\mid\mathcal F\}\\
&\le e^{nu(v-t)}\E[e^{-nu\widetilde{\mathcal S}}\mid\mathcal F]\\
&\le\exp\{nu(v-t)-(n/2)\log(1+2vu)\}.
\end{aligned}
\end{equation}
Using $\log(1+z)\ge z-z^2/2$ for $z\ge0$ in \eqref{holdout:chernoff-bound} and taking $u=t/(2v^2)$, we obtain
\begin{align*}
\P\{\widetilde{\mathcal S}<v-t\mid\mathcal F\}
&\le\exp\{-nut+nv^2u^2\}\\
&=\exp\{-nt^2/(4v^2)\}\qquad\text{almost surely}.\qedhere
\end{align*}
\end{proof}

\section{Recovery and Computational Guarantees}
\label{app:recovery-proofs}\label{app:completion-proofs}

\subsection{Auxiliary lemmas}
\label{app:auxiliary-lemmas}\label{aux:external-results}\label{aux:gaussian-tails-section}
Lemmas~\ref{aux:weighted-chi-square} and~\ref{lem:gaussian-products} control squared residual norms and inner products between independent predictor and error columns. Lemmas~\ref{aux:gaussian-matrix-deviation} and~\ref{lem:gaussian-maximum} control the designs uniformly over the sparse-regression cones. Lemma~\ref{aux:doob-maximal} controls the martingale maxima in the parent-recovery proof.

\begin{lemma}[{\citet[Lemma~1]{laurent2000adaptive}}]
\label{aux:weighted-chi-square}\label{aux:gaussian-tails}
Let $g_1,\ldots,g_r$ be independent standard Gaussian variables, and let $a\in[0,\infty)^r\setminus\{0\}$ be a deterministic vector of weights. For $Z=\sum_{i=1}^r a_i(g_i^2-1)$ and every $x>0$,
\begin{align}
\P\{Z\ge2\|a\|_2\sqrt{x}+2\|a\|_\infty x\}&\le e^{-x},\label{aux:weighted-upper-tail}\\
\P\{Z\le-2\|a\|_2\sqrt{x}\}&\le e^{-x}.\label{aux:weighted-lower-tail}
\end{align}
\end{lemma}

\begin{lemma}[Gaussian squares and products]\label{lem:gaussian-products}
Let $r\ge1$ be an integer, and let the $2r$ variables $g_1,\ldots,g_r,h_1,\ldots,h_r$ be mutually independent standard Gaussian variables. For every $x>0$,
\begin{align}
\P\left\{\sum_{i=1}^r g_i^2>r+2\sqrt{rx}+2x\right\}&\le e^{-x},\label{aux:gaussian-square-tail}\\
\P\left\{\sum_{i=1}^r g_i^2\le r-2\sqrt{rx}\right\}&\le e^{-x},\label{aux:chi-square-lower-tail}\\
\P\left\{\left|\sum_{i=1}^r g_ih_i\right|>2(\sqrt{rx}+x)\right\}&\le2e^{-x}.\label{ord:product-tail}
\end{align}
\end{lemma}
\begin{proof}
Apply \eqref{aux:weighted-upper-tail} and \eqref{aux:weighted-lower-tail} with $a=\mathbf1=(1,\ldots,1)^\top$ to obtain \eqref{aux:gaussian-square-tail} and \eqref{aux:chi-square-lower-tail}, respectively.

For \eqref{ord:product-tail}, set $Z=\sum_{i=1}^r g_ih_i$ and $g=(g_1,\ldots,g_r)^\top$. Equation~\eqref{aux:gaussian-square-tail} gives
\begin{equation}\label{aux:gaussian-norm-tail}
\P\{\|g\|_2>\sqrt r+\sqrt{2x}\}\le e^{-x}.
\end{equation}
Since $Z\mid g\sim N(0,\|g\|_2^2)$, the Gaussian tail bound yields
\begin{equation}\label{aux:conditional-product-tail}
\P\{|Z|>\sqrt{2x}\|g\|_2\mid g\}\le e^{-x}.
\end{equation}
Whenever $\|g\|_2\le\sqrt r+\sqrt{2x}$,
\[
\sqrt{2x}\|g\|_2\le\sqrt{2rx}+2x\le2(\sqrt{rx}+x).
\]
Combining \eqref{aux:gaussian-norm-tail} and \eqref{aux:conditional-product-tail} therefore gives
\begin{align*}
\P\{|Z|>2(\sqrt{rx}+x)\}
&\le\P\{\|g\|_2>\sqrt r+\sqrt{2x}\}
+\E\bigl[\P\{|Z|>\sqrt{2x}\|g\|_2\mid g\}\bigr]\\
&\le2e^{-x}.\qedhere
\end{align*}
\end{proof}

\begin{lemma}[{\citet[Theorem~13.17]{bandeira2026mathematics}}]
\label{aux:gaussian-matrix-deviation}
Let $r,m\ge1$ be integers, let $M\in\R^{r\times m}$ have independent $N(0,1)$ entries, and let $\mathcal T$ be a nonempty closed subset of the unit sphere in $\R^m$. Define
\[
a_r=\E\|h\|_2,\quad h\sim N(0,I_r),
\qquad
\omega(\mathcal T)=\E\sup_{u\in \mathcal T}\langle g,u\rangle,\quad g\sim N(0,I_m).
\]
For every $\varepsilon>\omega(\mathcal T)/a_r$, with probability at least
\[
1-2\exp\left\{-\frac{a_r^2}{2}
\left(\varepsilon-\frac{\omega(\mathcal T)}{a_r}\right)^2\right\},
\]
the following bounds hold simultaneously for all $u\in \mathcal T$:
\[
(1-\varepsilon)\|u\|_2
\le\left\|\frac1{a_r}Mu\right\|_2
\le(1+\varepsilon)\|u\|_2.
\]
Here $r^2/(r+1)\le a_r^2\le r$.
\end{lemma}

\begin{lemma}[Gaussian maximum]
\label{lem:gaussian-maximum}
Let $Z_1,\ldots,Z_N$ be centered Gaussian variables, not necessarily independent, with $N\ge1$ and $\max_j\E Z_j^2\le\sigma^2$ for some $\sigma\ge0$. Then
\begin{equation}
\E\max_{1\le j\le N}|Z_j|\le\sigma\sqrt{2\log(2N)}.
\label{aux:gaussian-maximum}
\end{equation}
\end{lemma}
\begin{proof}
For $t>0$, Jensen's inequality and $\E e^{\pm tZ_j}\le e^{t^2\sigma^2/2}$ give
\[
\begin{aligned}
\exp\!\left(t\,\E\max_{1\le j\le N}|Z_j|\right)
&\le\E\exp\!\left(t\max_{1\le j\le N}|Z_j|\right)\\
&\le\sum_{j=1}^N\E\bigl(e^{tZ_j}+e^{-tZ_j}\bigr)
\le2N e^{t^2\sigma^2/2}.
\end{aligned}
\]
Taking logarithms and dividing by $t$ yields
\[
\E\max_{1\le j\le N}|Z_j|
\le\frac{\log(2N)}t+\frac{t\sigma^2}2.
\]
Taking the infimum over $t>0$ proves the claim.
\end{proof}

\begin{lemma}[Doob's inequality, {\citet[Theorem~4.4.2]{durrett2019probability}}]
\label{aux:doob-maximal}
Let $X_m$ be a submartingale, $\overline X_n=\max_{0\le m\le n}X_m^+$, $\lambda>0$, and $A=\{\overline X_n\ge\lambda\}$. Then
\begin{equation}\label{aux:doob-inequality}
\lambda\P(A)\le\E[X_n\mathbf1_A]\le\E X_n^+.
\end{equation}
Here $x^+=\max\{x,0\}$ and $\mathbf1_A$ is the indicator of $A$.
\end{lemma}

\subsection{Proof of Theorem~\ref{thm:ordering-recovery}}
\label{app:ordering-proof}
\subsubsection{Setup}
\label{aux:ordering-section}
If $d=0$, every ordering is compatible. Assume $d\ge1$ and keep the penalty $\lambda_{\mathrm{ord}}>0$ deterministic. Let $\mathbf X\in\R^{n\times p}$ have independent $N(0,\Sigma)$ rows, and let $\mathbf W\in\R^{n\times p}$ be the structural error matrix, with independent columns $\mathbf W_j\sim N(0,\sigma_j^2I_n)$. Condition throughout on the data-independent split $I_1\cup I_2=[n]$. Without loss of generality, assume that $n$ is even, so $|I_1|=|I_2|=n/2$. Odd $n$ changes only numerical constants. Let $\bar\ell$ denote the other fold. Both fold designs have independent $N(0,\Sigma)$ rows, and the two folds are independent. Applying \eqref{eq:linear-sem} and \eqref{eq:error-path} to all observations gives
\begin{align}
\mathbf X_j&=\sum_{a\in\pa(j)}B_{aj}\mathbf X_a+\mathbf W_j,\label{ord:sample-sem}\\
\mathbf X_j&=\sum_{a\in\an(j)}B_{a\pathto j}\mathbf W_a+\mathbf W_j\qquad(j\in V).\label{ord:sample-error-path}
\end{align}

For an ancestral set $A\subseteq V$, restriction of \eqref{ord:sample-sem} gives $\mathbf X_A(I_{|A|}-B_{AA})=\mathbf W_A$. Acyclicity makes $I_{|A|}-B_{AA}$ invertible, so
\begin{equation}\label{ord:ancestral-data}
\mathbf X_A=\mathbf W_A(I_{|A|}-B_{AA})^{-1}.
\end{equation}
For each node $f\in V$, let $\mathcal D_f$ contain $f$ and its descendants, and let $U_f$ be its complement:
\[
\mathcal D_f=\{f\}\cup\{k\in V:f\in\an(k)\},
\qquad U_f=V\setminus\mathcal D_f.
\]
The set $U_f$ is ancestral, since a parent outside $U_f$ would make its child a descendant of $f$. Thus \eqref{ord:ancestral-data} shows that $\mathbf X_{U_f}$ is determined by $\mathbf W_{U_f}$.

For a fixed predictor set $S\subsetneq V$, response node $j\notin S$, and fold $\ell\in\{1,2\}$, define the Lasso fit by
\begin{equation}\label{ord:fixed-set-lasso}
\widehat\theta_{Sj}^{(\ell)}\in\argmin_{\theta\in\R^{|S|}}
\left\{\frac{\|\mathbf X_{I_\ell j}-\mathbf X_{I_\ell S}\theta\|_2^2}{n}
+\lambda_{\mathrm{ord}}\|\theta\|_1\right\}.
\end{equation}
For nonempty $S$, positive definiteness of $\Sigma_{SS}$ gives the fitting design $\mathbf X_{I_\ell S}$ a nondegenerate joint Gaussian density, so its columns are in general position almost surely. For such designs, the minimizer in \eqref{ord:fixed-set-lasso} is unique for every response \citep[Lemma~3]{tibshirani2013lasso}. Since $\lambda_{\mathrm{ord}}>0$, comparison with the zero vector gives $\lambda_{\mathrm{ord}}\|\widehat\theta_{Sj}^{(\ell)}\|_1\le\|\mathbf X_{I_\ell j}\|_2^2/n$, so minimizers are locally bounded as the fitting data vary. Joint continuity of the objective and uniqueness then imply continuity of the minimizer on the Borel set of general-position designs. Extending the fit by zero elsewhere gives a version measurable with respect to $\sigma(\mathbf X_{I_\ell S},\mathbf X_{I_\ell j})$. Take the empty vector when $S=\emptyset$. A finite intersection over predictor sets, response nodes, and folds shows that these versions agree with all HOST fits on one event of probability one. Define its loss on the other fold and the combined score by
\begin{equation}\label{ord:fold-score}
\widehat L_{j,S}^{(\ell)}
=\frac{2\|\mathbf X_{I_{\bar\ell}j}-\mathbf X_{I_{\bar\ell}S}\widehat\theta_{Sj}^{(\ell)}\|_2^2}{n},
\qquad
\widehat{\mathcal S}_{j,S}=\frac12\sum_{\ell=1}^2\widehat L_{j,S}^{(\ell)}.
\end{equation}
Write $\widehat S_r$ for the nodes selected before step $r$, with $\widehat S_1=\emptyset$. For $r\ge2$, equations~\eqref{ord:fixed-set-lasso}--\eqref{ord:fold-score} at $S=\widehat S_r$ give $\widehat\theta_{\widehat S_r,j}^{(\ell)}=\widehat\theta_j^{(\ell)}(r)$ and $\widehat{\mathcal S}_{j,\widehat S_r}=\widehat{\mathcal S}_j(r)$ from \eqref{eq:algordering}--\eqref{eq:algscores}. At $S=\emptyset$, \eqref{ord:fold-score} gives $\widehat{\mathcal S}_{j,\emptyset}=\|\mathbf X_j\|_2^2/n=\widehat\Sigma_{jj}$, so its minimizer agrees with the initialization of Algorithm~\ref{alg1}.

\subsubsection{Upper bounds for remaining-source scores}
\label{aux:gaussian-covariance-section}
For $1\le d\le p$, define the union of Lasso cones
\[
 \mathcal C_d=\bigcup_{\substack{T\subseteq[p]\\|T|\le d}}
 \{v\in\R^p:\|v_{T^c}\|_1\le3\|v_T\|_1\}.
\]

\begin{lemma}[Covariance bounds for the fold designs]
\label{aux:cone-covariance}
Fix a sufficiently large universal constant $C_{\mathrm{cov}}$. Under the setup in Section~\ref{aux:ordering-section}, for $p\ge2$ and $1\le d\le p$, the condition
\begin{equation}
 n\ge C_{\mathrm{cov}}\frac{d\Sigma_{\max}}{\lambda_{\min}(\Sigma)}\log(2p)
\label{ord:cone-size}
\end{equation}
implies that, with probability at least $1-4p^{-8}$,
\begin{equation}
 \frac14 v^\top\Sigma v
 \le\frac{2\|\mathbf X_{I_\ell}v\|_2^2}{n}
 \le\frac94v^\top\Sigma v
 \quad(v\in\mathcal C_d,\ \ell=1,2).
\label{ord:cone-event}
\end{equation}
\end{lemma}

\begin{proof}
Write $\kappa=\Sigma_{\max}/\lambda_{\min}(\Sigma)\ge1$ and $\mathcal T=\{v\in\mathcal C_d:v^\top\Sigma v=1\}$. Every nonzero $v\in\mathcal C_d$ satisfies $v/\sqrt{v^\top\Sigma v}\in \mathcal T$, so it suffices to prove \eqref{ord:cone-event} on $\mathcal T$. For each such $v$, choose $J\subseteq[p]$ with $|J|\le d$ and $\|v_{J^c}\|_1\le3\|v_J\|_1$. The cone condition and Cauchy--Schwarz give
\[
\|v\|_1\le4\|v_J\|_1
\le4\sqrt d\,\|v\|_2
\le4\sqrt{\frac d{\lambda_{\min}(\Sigma)}},
\]
where the last inequality follows from $1=v^\top\Sigma v\ge\lambda_{\min}(\Sigma)\|v\|_2^2$. For $g\sim N(0,I_p)$, each coordinate of $\Sigma^{1/2}g$ has variance at most $\Sigma_{\max}$. The H\"older inequality and Lemma~\ref{lem:gaussian-maximum} therefore yield
\[
\omega(\Sigma^{1/2}\mathcal T)
\le4\sqrt{\frac d{\lambda_{\min}(\Sigma)}}\,\E\|\Sigma^{1/2}g\|_\infty
\le4\sqrt{2d\kappa\log(2p)}.
\]
The matrix $\mathbf X_{I_\ell}\Sigma^{-1/2}$ has independent standard Gaussian entries, and $\Sigma^{1/2}\mathcal T$ is a closed subset of the unit sphere. Apply Lemma~\ref{aux:gaussian-matrix-deviation} with
\[
\varepsilon=\frac{\omega(\Sigma^{1/2}\mathcal T)+4\sqrt{\log p}}{a_{n/2}}.
\]
Since $r^2/(r+1)\ge r-1$, the lemma gives $\sqrt{n/2-1}\le a_{n/2}\le\sqrt{n/2}$. Taking a union bound over the two folds then gives, with probability at least $1-4p^{-8}$, simultaneously for $v\in \mathcal T$ and $\ell=1,2$,
\[
\begin{aligned}
\|\mathbf X_{I_\ell}v\|_2
&\ge\sqrt{n/2-1}-4\sqrt{2d\kappa\log(2p)}-4\sqrt{\log p},\\
\|\mathbf X_{I_\ell}v\|_2
&\le\sqrt{n/2}+4\sqrt{2d\kappa\log(2p)}+4\sqrt{\log p}.
\end{aligned}
\]
For sufficiently large $C_{\mathrm{cov}}$, condition~\eqref{ord:cone-size} implies
\[
\bigl(4\sqrt{2d\kappa\log(2p)}+4\sqrt{\log p}\bigr)^2
\le96d\kappa\log(2p)\le\frac{n/2-4}{4}.
\]
Consequently,
\[
\begin{aligned}
\|\mathbf X_{I_\ell}v\|_2
&\ge\sqrt{n/2-1}-\frac12\sqrt{n/2-4}
\ge\frac12\sqrt{n/2},\\
\|\mathbf X_{I_\ell}v\|_2
&\le\sqrt{n/2}+\frac12\sqrt{n/2-4}
\le\frac32\sqrt{n/2}.
\end{aligned}
\]
Squaring and rescaling from $\mathcal T$ to $\mathcal C_d$ proves \eqref{ord:cone-event}.
\end{proof}

\begin{lemma}[Uniform upper bounds for remaining-source scores]
\label{aux:source-scores}
Fix sufficiently large universal constants $C_{\mathrm{ord}}$ and $C_{\mathrm L}$. Suppose $d\ge1$, condition~\eqref{ord:cone-size} holds, and
\begin{equation}\label{ord:source-penalty}
 \lambda_{\mathrm{ord}}\ge C_{\mathrm{ord}}\sigma_{\max}\sqrt{\frac{\Sigma_{\max}\log p}{n}}.
\end{equation}
With probability at least $1-6p^{-6}$, every ancestral $S\subsetneq V$, every $j\in\Src(S)$, and both $\ell\in\{1,2\}$ satisfy
\begin{equation}
 \widehat L_{j,S}^{(\ell)}
 \le\sigma_j^2\left(1+8\sqrt{\frac{\log p}{n}}+\frac{32\log p}{n}\right)
  +\frac{C_{\mathrm L}d\lambda_{\mathrm{ord}}^2}{\lambda_{\min}(\Sigma)}.
 \label{aux:source-score-bound}
\end{equation}
Inequality~\eqref{aux:source-score-bound} also holds with $\widehat L_{j,S}^{(\ell)}$ replaced by $\widehat{\mathcal S}_{j,S}$.
\end{lemma}

\begin{proof}
Lemma~\ref{aux:cone-covariance} gives \eqref{ord:cone-event} with failure probability $4p^{-8}$. If $a\notin\mathcal D_j$, then $j\notin\{a\}\cup\an(a)$. The error representation \eqref{ord:sample-error-path} and the independence of the columns of $\mathbf W$ therefore imply that $\mathbf X_a$ is independent of $\mathbf W_j$.

For independent $g,h\sim N(0,I_{n/2})$, the Gaussian fold distributions give
\[
\begin{aligned}
\frac{\mathbf X_{I_\ell a}^\top \mathbf W_{I_\ell j}}{\sigma_j\sqrt{\Sigma_{aa}}}
&\overset{d}{=}g^\top h=\sum_{i=1}^{n/2}g_i h_i
\qquad(a\notin\mathcal D_j),\\
\frac{\|\mathbf W_{I_\ell j}\|_2^2}{\sigma_j^2}
&\overset{d}{=}\|h\|_2^2=\sum_{i=1}^{n/2}h_i^2
\qquad(j\in V).
\end{aligned}
\]
Apply \eqref{ord:product-tail} and \eqref{aux:gaussian-square-tail} with $r=n/2$ and $x=8\log p$. Using $n/2\ge8\log p$ from \eqref{ord:cone-size} and the penalty condition \eqref{ord:source-penalty} gives, for fixed indices, with respective failure probabilities $2p^{-8}$ and $p^{-8}$,
\begin{align}
\frac{2|\mathbf X_{I_\ell a}^\top \mathbf W_{I_\ell j}|}{n}
&\le2\sigma_j\sqrt{\Sigma_{aa}}\left(4\sqrt{\frac{\log p}{n}}+\frac{16\log p}{n}\right)\notag\\
&\le16\sigma_{\max}\sqrt{\frac{\Sigma_{\max}\log p}{n}}\le\lambda_{\mathrm{ord}}/2
\quad(a\notin\mathcal D_j),\label{ord:noise-event}\\
\frac{2\|\mathbf W_{I_\ell j}\|_2^2}{n}
&\le\sigma_j^2\left(1+8\sqrt{\frac{\log p}{n}}+\frac{32\log p}{n}\right)
\quad(j\in[p]).\label{ord:innovation-event}
\end{align}
A union bound over both folds and at most $p^2$ pairs $(a,j)$ in \eqref{ord:noise-event}, and $p$ nodes in \eqref{ord:innovation-event}, gives respective failure probabilities $4p^{-6}$ and $2p^{-7}$.

Work on the event where \eqref{ord:cone-event}, \eqref{ord:noise-event}, and \eqref{ord:innovation-event} all hold. These three events do not depend on the predictor set $S$. Fix an ancestral $S$, a remaining source $j\in\Src(S)$, and a fitting fold $\ell$. Since $S$ is ancestral and $j\notin S$, it contains no descendant of $j$. Thus $T=\pa(j)\subseteq S\subseteq U_j$ and $|T|\le d$. Define the estimation error $\Delta\in\R^p$ by
\[
\Delta_S=\widehat\theta_{Sj}^{(\ell)}-B_{Sj},
\qquad \Delta_{S^c}=0.
\]
The true coefficient vector $B_{\cdot j}$ is supported on $T\subseteq S$. By \eqref{ord:sample-sem}, optimality of the Lasso fit gives
\[
\begin{aligned}
&\frac{\|\mathbf W_{I_\ell j}-\mathbf X_{I_\ell}\Delta\|_2^2}{n}
+\lambda_{\mathrm{ord}}\|B_{\cdot j}+\Delta\|_1\\
&=\frac{\|\mathbf W_{I_\ell j}\|_2^2}{n}
-\frac{2\mathbf W_{I_\ell j}^\top\mathbf X_{I_\ell}\Delta}{n}
+\frac{\|\mathbf X_{I_\ell}\Delta\|_2^2}{n}
+\lambda_{\mathrm{ord}}\|B_{\cdot j}+\Delta\|_1\\
&\le\frac{\|\mathbf W_{I_\ell j}\|_2^2}{n}
+\lambda_{\mathrm{ord}}\|B_{\cdot j}\|_1.
\end{aligned}
\]
Since $\Delta$ is supported on $S\subseteq U_j$ and $B_{T^c,j}=0$, equation~\eqref{ord:noise-event} and the triangle inequality give
\[
\begin{aligned}
\frac{\|\mathbf X_{I_\ell}\Delta\|_2^2}{n}
&\le\frac{2\mathbf W_{I_\ell j}^\top\mathbf X_{I_\ell}\Delta}{n}
+\lambda_{\mathrm{ord}}\bigl(\|B_{\cdot j}\|_1-\|B_{\cdot j}+\Delta\|_1\bigr)\\
&\le\frac{\lambda_{\mathrm{ord}}}{2}\|\Delta\|_1
+\lambda_{\mathrm{ord}}\bigl(\|\Delta_T\|_1-\|\Delta_{T^c}\|_1\bigr).
\end{aligned}
\]
Rearranging yields the basic inequality
\begin{equation}
\frac{2\|\mathbf X_{I_\ell}\Delta\|_2^2}{n}
+\lambda_{\mathrm{ord}}\|\Delta_{T^c}\|_1
\le3\lambda_{\mathrm{ord}}\|\Delta_T\|_1.
\label{aux:lasso-basic-inequality}
\end{equation}
Hence $\Delta\in\mathcal C_d$. The lower bound in \eqref{ord:cone-event} and Cauchy--Schwarz give
\[
\frac{\|\Sigma^{1/2}\Delta\|_2^2}{4}
\le\frac{2\|\mathbf X_{I_\ell}\Delta\|_2^2}{n}
\le3\lambda_{\mathrm{ord}}\|\Delta_T\|_1
\le3\lambda_{\mathrm{ord}}\sqrt{\frac{|T|}{\lambda_{\min}(\Sigma)}}\,\|\Sigma^{1/2}\Delta\|_2.
\]
The cone inequality $\|\Delta_{T^c}\|_1\le3\|\Delta_T\|_1$ from \eqref{aux:lasso-basic-inequality} also gives
\[
\|\Delta\|_1\le4\|\Delta_T\|_1
\le4\sqrt{\frac{|T|}{\lambda_{\min}(\Sigma)}}\,\|\Sigma^{1/2}\Delta\|_2.
\]
Consequently,
\begin{equation}\label{ord:lasso-error-bounds}
\|\Sigma^{1/2}\Delta\|_2^2
\le\frac{12^2|T|\lambda_{\mathrm{ord}}^2}{\lambda_{\min}(\Sigma)},
\qquad
\lambda_{\mathrm{ord}}\|\Delta\|_1
\le\frac{4\cdot12|T|\lambda_{\mathrm{ord}}^2}{\lambda_{\min}(\Sigma)}.
\end{equation}
On the other fold, expanding the residual norm and applying \eqref{ord:noise-event}, the upper bound in \eqref{ord:cone-event}, and \eqref{ord:lasso-error-bounds} gives
\begin{equation}
\begin{aligned}[b]
\widehat L_{j,S}^{(\ell)}
&=\frac{2\|\mathbf W_{I_{\bar\ell}j}-\mathbf X_{I_{\bar\ell}}\Delta\|_2^2}{n}\\
&\le\frac{2\|\mathbf W_{I_{\bar\ell}j}\|_2^2}{n}
+\frac94\|\Sigma^{1/2}\Delta\|_2^2+\lambda_{\mathrm{ord}}\|\Delta\|_1\\
&\le\frac{2\|\mathbf W_{I_{\bar\ell}j}\|_2^2}{n}
+\frac{C_{\mathrm L}|T|\lambda_{\mathrm{ord}}^2}{\lambda_{\min}(\Sigma)}.
\end{aligned}
\label{aux:lasso-holdout-bound}
\end{equation}
Using \eqref{ord:innovation-event} and $|T|\le d$ in \eqref{aux:lasso-holdout-bound} proves \eqref{aux:source-score-bound}. If $S=\emptyset$, then $T=\emptyset$ and $\Delta=0$, so \eqref{aux:lasso-holdout-bound} still holds. Averaging the two fold bounds according to \eqref{ord:fold-score} proves \eqref{aux:source-score-bound} for $\widehat{\mathcal S}_{j,S}$. The events in \eqref{ord:cone-event}, \eqref{ord:noise-event}, and \eqref{ord:innovation-event} establish \eqref{aux:source-score-bound} for every ancestral $S$ and remaining source $j$, so no union over predictor sets is needed. Adding their failure probabilities gives
\[
4p^{-8}+4p^{-6}+2p^{-7}\le6p^{-6}.\qedhere
\]
\end{proof}

\subsubsection{Lower bounds for remaining-nonsource scores}
\label{aux:lasso-section}

\begin{lemma}[Conditional lower tail for a selected predictor set]
\label{aux:auxiliary-score-lower-tail}
Fix a node $f\in V$ and a proper descendant $k$ of $f$. Let $S\subseteq U_f$ be a random predictor set measurable with respect to $\mathbf W_{U_f}$. For each fold $\ell\in\{1,2\}$, define
\[
\mathcal F_{f,\ell}=\sigma(\mathbf W_{U_f},\mathbf W_{I_\ell\mathcal D_f}).
\]
For every $0<t<\mathcal S_k(U_f)$, almost surely,
\begin{equation}\label{ord:conditional-score-tail}
\P\left\{\widehat L_{k,S}^{(\ell)}
               <\mathcal S_k(U_f)-t\mid\mathcal F_{f,\ell}\right\}
 \le\exp\left\{-\frac{nt^2}{8\mathcal S_k(U_f)^2}\right\}.
\end{equation}
\end{lemma}

\begin{proof}
Since $S$ is $\mathbf W_{U_f}$-measurable, \eqref{ord:ancestral-data} makes both fold designs on $S$ measurable with respect to $\mathbf W_{U_f}$. The fitting-fold errors $\mathbf W_{I_\ell V}$ are also $\mathcal F_{f,\ell}$-measurable and determine $\mathbf X_{I_\ell k}$ through \eqref{ord:sample-error-path}. Thus the minimizer rule fixed in Section~\ref{aux:ordering-section} makes $\widehat\theta_{Sk}^{(\ell)}$ and $\mathbf X_{I_{\bar\ell}S}\widehat\theta_{Sk}^{(\ell)}$ $\mathcal F_{f,\ell}$-measurable.

By \eqref{ord:sample-error-path} and Proposition~\ref{prop:ancestral-regression}, the held-out response satisfies
\begin{equation}\label{ord:heldout-response}
\mathbf X_{I_{\bar\ell}k}
=\mathbf X_{I_{\bar\ell}U_f}\theta_k(U_f)+\mathbf W_{I_{\bar\ell}k}
+\sum_{a\in\mathcal D_f\setminus\{k\}}B_{a\pathto k}\mathbf W_{I_{\bar\ell}a}.
\end{equation}
The errors $\mathbf W_{I_{\bar\ell}a}$ for $a\in\mathcal D_f$ are mutually independent and independent of $\mathcal F_{f,\ell}$, because observations in the two folds and errors at different nodes are independent. The error term in \eqref{ord:heldout-response} has conditional distribution
\[
\mathbf W_{I_{\bar\ell}k}
+\sum_{a\in\mathcal D_f\setminus\{k\}}B_{a\pathto k}\mathbf W_{I_{\bar\ell}a}
\,\bigm|\,\mathcal F_{f,\ell}
\sim N\!\left(0,\mathcal S_k(U_f)I_{n/2}\right)
\]
by \eqref{eq:population-score-path}. Apply Lemma~\ref{lem:holdout-lower-tail} to \eqref{ord:heldout-response}, with design $\mathbf X_{I_{\bar\ell}U_f}$, coefficient $\theta_k(U_f)$, and fitted coefficient equal to $\widehat\theta_{Sk}^{(\ell)}$ on $S$ and zero on $U_f\setminus S$. This design and these coefficients are $\mathcal F_{f,\ell}$-measurable. By \eqref{ord:fold-score}, the resulting loss in dimension $n/2$ is $\widehat L_{k,S}^{(\ell)}$, so the conclusion of Lemma~\ref{lem:holdout-lower-tail} is \eqref{ord:conditional-score-tail}.
\end{proof}

\subsubsection{Completion of the proof}
We prove recovery by induction on the number of selected nodes. Let
\[
\mathcal A_r
=\bigcap_{s=1}^r\{\pa(\widehat\pi_s)\subseteq\widehat S_s\},
\qquad r=0,\ldots,p,
\]
be the event that the first $r$ choices are all remaining sources, with $\mathcal A_0$ the sure event. On $\mathcal A_{r-1}$, the set $\widehat S_r$ is ancestral. Define the source upper-bound event by
\begin{equation}
\mathcal E_{\mathrm{src}}
=\bigcap_{\substack{S\subsetneq V\\S\text{ ancestral}}}
\bigcap_{a\in\Src(S)}
\left\{\widehat{\mathcal S}_{a,S}\le\sigma_a^2+\tau_\star/4\right\}.
\label{ord:source-score-event}
\end{equation}
First we bound $\P(\mathcal E_{\mathrm{src}}^c)$. Use the fixed multiplier $C_{\mathrm{ord}}$ in Lemma~\ref{aux:source-scores} and choose $\lambda_{\mathrm{ord}}=C_{\mathrm{ord}}\sigma_{\max}\sqrt{\Sigma_{\max}\log p/n}$, as prescribed in Theorem~\ref{thm:ordering-recovery}. The first term of \eqref{eq:ordering-rate}, with a sufficiently large fixed factor, ensures the design condition \eqref{ord:cone-size} and
\[
\frac{C_{\mathrm L}d\lambda_{\mathrm{ord}}^2}{\lambda_{\min}(\Sigma)}
=C_{\mathrm L}C_{\mathrm{ord}}^2
\frac{d\Sigma_{\max}}{\lambda_{\min}(\Sigma)}
\frac{\sigma_{\max}^2\log p}{n}
\le\frac{\tau_\star}{8}.
\]
The second term of \eqref{eq:ordering-rate} controls the fluctuation of the error norms. Since $\sigma_j^2\le\Sigma_{jj}\le\Sigma_{\max}$, it ensures
\[
\sigma_j^2\left(8\sqrt{\frac{\log p}{n}}+\frac{32\log p}{n}\right)
\le\Sigma_{\max}\left(8\sqrt{\frac{\log p}{n}}+\frac{32\log p}{n}\right)
\le\frac{\tau_\star}{8}.
\]
Lemma~\ref{aux:source-scores} therefore gives
\begin{equation}
\P(\mathcal E_{\mathrm{src}}^c)\le6p^{-6}.
\label{ord:source-event-probability}
\end{equation}

We next bound the probability that a correct prefix gains a remaining nonsource, then iterate over $r$.

\begin{lemma}[One-step ordering error]\label{aux:one-step-ordering}
\normalfont
Under the sample-size and penalty conditions of Theorem~\ref{thm:ordering-recovery}, every $r\in[p]$ satisfies
\begin{equation}
\P(\mathcal E_{\mathrm{src}}\cap\mathcal A_{r-1}\cap\mathcal A_r^c)
\le 2p^{-6}.
\label{ord:one-step-error}
\end{equation}
\end{lemma}
\begin{proof}
Suppose the first $r-1$ choices are remaining sources, so the selected set $\widehat S_r$ is ancestral. For any remaining nonsource $k$, following unselected parents backwards from $k$ reaches a remaining source $f$. Since $f$ is unselected and $\widehat S_r$ is ancestral, none of its descendants has been selected either. Thus all previous choices lie in $U_f$, giving $\widehat S_r\subseteq U_f$ and $|U_f|\ge r-1$.

Fix a node $f$ with $|U_f|\ge r-1$. We now determine the value of $\widehat S_r$ whenever $\widehat S_r\subseteq U_f$, using only the data $\mathbf X_{U_f}$. At a specified selected set $T\subseteq U_f$, every candidate $a\in U_f\setminus T$ uses a response and predictors from $\mathbf X_{U_f}$, so its fits and score depend only on these data. If the next HOST choice is in $U_f$, it must be the smallest-labeled node in
\begin{equation}\label{ord:within-Uf-choice}
\argmin_{a\in U_f\setminus T}\widehat{\mathcal S}_{a,T}.
\end{equation}
Starting from the empty set, successive minimization in \eqref{ord:within-Uf-choice} determines each choice as long as HOST selects within $U_f$. The scores of candidates in $V\setminus U_f$ can determine when HOST first leaves $U_f$, but cannot change the choices preceding that exit.

By the measurability properties in Section~\ref{aux:ordering-section}, applying \eqref{ord:within-Uf-choice} $r-1$ times defines a $\mathbf W_{U_f}$-measurable set
\[
S_{r,f}=S_{r,f}(\mathbf W_{U_f})\subseteq U_f,
\qquad |S_{r,f}|=r-1.
\]
Thus fixing $r$, $f$, and $\mathbf W_{U_f}$ fixes $S_{r,f}$. This set agrees with the actual selected set $\widehat S_r$ whenever $\widehat S_r\subseteq U_f$:
\begin{equation}\label{ord:selected-set-identification}
\widehat S_r\subseteq U_f\quad\Longrightarrow\quad
\widehat S_r=S_{r,f}
\qquad\text{almost surely}.
\end{equation}
Now fix $k\in\mathcal D_f\setminus\{f\}$. On $\mathcal E_{\mathrm{src}}\cap\mathcal A_{r-1}$, if HOST selects $k$ while $f$ is a remaining source, then \eqref{ord:selected-set-identification} gives $\widehat S_r=S_{r,f}$, and score minimization gives
\begin{equation}\label{ord:error-score-comparison}
\widehat{\mathcal S}_{k,S_{r,f}}
\le\widehat{\mathcal S}_{f,S_{r,f}}
\le\sigma_f^2+\tau_\star/4
<\sigma_f^2+3\tau_\star/4.
\end{equation}
Equation~\eqref{ord:error-score-comparison} shows that selecting $k$ under these conditions requires $\widehat{\mathcal S}_{k,S_{r,f}}<\sigma_f^2+3\tau_\star/4$. We now bound the unconditional probability of this score inequality.

With $U_f$ as the predictor set, all parents of $f$ are included, whereas a parent of $k$ on a path from $f$ lies outside it. Thus $f$ is the unique remaining source at $U_f$, and every proper descendant of $f$, including $k$, is a remaining nonsource. Equation~\eqref{eq:population-score-path} gives $\mathcal S_f(U_f)=\sigma_f^2$, and choosing the zero coefficient vector gives $\mathcal S_k(U_f)\le\Sigma_{kk}$. Assumption~\ref{assu:taumin} therefore implies
\begin{equation}\label{ord:descendant-score-margin}
\sigma_f^2+\tau_\star\le\mathcal S_k(U_f)\le\Sigma_{kk}\le\Sigma_{\max}.
\end{equation}
Since $S_{r,f}$ is $\mathbf W_{U_f}$-measurable, \eqref{ord:conditional-score-tail} applies with $S=S_{r,f}$ and $t=\tau_\star/4$. The combined score $\widehat{\mathcal S}_{k,S_{r,f}}$ averages the two fold losses by \eqref{ord:fold-score}, and \eqref{ord:descendant-score-margin} gives $\sigma_f^2+3\tau_\star/4\le\mathcal S_k(U_f)-\tau_\star/4$. Taking expectations in \eqref{ord:conditional-score-tail} therefore yields
\begin{equation}
\begin{aligned}[b]
&\P\{\widehat{\mathcal S}_{k,S_{r,f}}<\sigma_f^2+3\tau_\star/4\}\\
&\quad\le\sum_{\ell=1}^2\E\!\left[
\P\!\left\{\widehat L_{k,S_{r,f}}^{(\ell)}<\mathcal S_k(U_f)-\tau_\star/4
\,\middle|\,\mathcal F_{f,\ell}\right\}\right]\\
&\quad\le 2\exp\!\left\{-\frac{n\tau_\star^2}{128\Sigma_{\max}^2}\right\}
\le 2p^{-8}.
\end{aligned}
\label{ord:lower-event-probability}
\end{equation}
The bound by $2p^{-8}$ in \eqref{ord:lower-event-probability} follows from the second term of \eqref{eq:ordering-rate}, with its fixed factor sufficiently large. Every error on $\mathcal E_{\mathrm{src}}\cap\mathcal A_{r-1}$ yields a remaining source $f$ and a selected descendant $k$ satisfying \eqref{ord:error-score-comparison}. Summing \eqref{ord:lower-event-probability} over at most $p^2$ such pairs gives $p^2\cdot2p^{-8}=2p^{-6}$, proving \eqref{ord:one-step-error}.
\end{proof}

Finally, $\mathcal A_r\subseteq\mathcal A_{r-1}$, so Lemma~\ref{aux:one-step-ordering} gives, for $r=1,\ldots,p$,
\begin{equation}\label{ord:ordering-induction}
\begin{aligned}[b]
\P(\mathcal E_{\mathrm{src}}\cap\mathcal A_r^c)
&\le\P(\mathcal E_{\mathrm{src}}\cap\mathcal A_{r-1}^c)+2p^{-6}\\
&\le 2r p^{-6},
\end{aligned}
\end{equation}
where the last inequality follows by induction from $\P(\mathcal E_{\mathrm{src}}\cap\mathcal A_0^c)=0$. Since $\mathcal A_p$ is exactly the event that $\widehat\pi$ is compatible with $G$, combining \eqref{ord:source-event-probability} and \eqref{ord:ordering-induction} at $r=p$ gives
\[
\begin{aligned}
\P\{\widehat\pi\text{ is incompatible with }G\}
&\le\P(\mathcal E_{\mathrm{src}}^c)
+\P(\mathcal E_{\mathrm{src}}\cap\mathcal A_p^c)\\
&\le 6p^{-6}+2p^{-5}\le p^{-2}.
\end{aligned}
\]

\subsection{Proof of Theorem~\ref{thm:parent-recovery}}
\label{app:parent-proof}\label{app:parent-tools}

We first control the fixed true-parent fits. A direct expansion of the reconstruction error then gives martingale increments, whose exponential bounds yield a uniform reconstruction event. On this event, thresholding and induction recover all parents along every compatible ordering.

\subsubsection{Setup}
\label{res:setup}\label{res:appendix}

Use the data matrix $\mathbf X$ and structural errors $\mathbf W$ from Section~\ref{aux:ordering-section}. Fix one deterministic ordering compatible with $G$ and label $V=[p]$ accordingly throughout this section. Write $\mathcal F_k=\sigma(\mathbf W_1,\ldots,\mathbf W_k)$, with $\mathcal F_0$ trivial. This ordering fixes the order in which we expose the structural errors; the final recovery argument applies to any compatible ordering supplied to HOST. Until that argument, all local fits use the true parent sets:
\begin{equation}\label{res:local-fits}
\begin{gathered}[b]
D_k=\mathbf X_{\pa(k)},\qquad P_k=D_kD_k^\dagger,\\
\widehat B^{\mathrm{LS}}_{\pa(k),k}=D_k^\dagger\mathbf X_k,
\qquad\widehat B^{\mathrm{LS}}_{V\setminus\pa(k),k}=0,\\
\widehat r_k=(I_n-P_k)\mathbf X_k=(I_n-P_k)\mathbf W_k.
\end{gathered}
\end{equation}
For an empty design, $P_k=0$. The sample size assumptions below ensure $n>d$, so every nonempty $D_k$ has full column rank almost surely. For $i\in\pa(k)$, put $h_{ik}=D_k[(D_k^\top D_k)^{-1}]_{\cdot i}$, with columns indexed by $\pa(k)$, and put $h_{ik}=0$ otherwise. Every parent precedes its child, so $D_k,P_k,h_{ik}$ are $\mathcal F_{k-1}$-measurable and $\mathbf W_k$ is independent of $\mathcal F_{k-1}$. The OLS error identity and total-effect estimates are as follows, with any zero-denominator ratio defined as zero:
\begin{align}
\widehat B^{\mathrm{LS}}_{ik}-B_{ik}&=h_{ik}^\top\mathbf W_k,
\qquad P_kh_{ik}=h_{ik},\label{res:local-coefficient-error}\\
\widehat B_{k\pathto j}&=\frac{\widehat r_k^\top\mathbf X_j}{\|\widehat r_k\|_2^2}
\qquad(k\ne j).\label{res:total-effect}
\end{align}

We use the following covariance relations. For $i\in T\subseteq V$, additional Gaussian conditioning can only reduce conditional variance, so
\begin{equation}\label{res:precision-bounds}
(\Sigma_{TT}^{-1})_{ii}
=\frac{1}{\Var(X_i\mid X_{T\setminus\{i\}})}
\le\frac{1}{\Var(X_i\mid X_{V\setminus\{i\}})}=\Omega_{ii}.
\end{equation}
The SEM also gives $\Omega=(I-B)\diag(\sigma_k^{-2}:k\in V)(I-B)^\top$, hence
\begin{equation}\label{res:precision-energy}
\Omega_{ii}=\sum_{k\in V}\frac{(\mathbf1\{k=i\}-B_{ik})^2}{\sigma_k^2},
\end{equation}
where $\mathbf1\{\cdot\}$ denotes an indicator.

\subsubsection{Events for the fixed-parent estimates}

Fix a sufficiently large universal constant $C_{\mathrm{par}}$. We use two events for the OLS coefficients and total-effect estimates in \eqref{res:local-fits}--\eqref{res:total-effect}. The first event controls the local fits. For each node $k$, let
\begin{equation}\label{par:design-event}
\mathcal E_k=\{D_k^\top D_k\succeq n\Sigma_{\pa(k),\pa(k)}/4\},
\end{equation}
with $\mathcal E_k$ the whole space for an empty design. Since $D_k$ is $\mathcal F_{k-1}$-measurable, so is $\mathcal E_k$. Define
\begin{equation}\label{par:local-good-event}
\mathcal E_{\mathrm{LS}}
=\bigcap_{k\in V}\left(\mathcal E_k\cap
\{\|\widehat r_k\|_2^2\ge n\sigma_k^2/2\}\right)
\cap\bigcap_{(i,k)\in E}
\{ |\widehat B^{\mathrm{LS}}_{ik}-B_{ik}|\le|B_{ik}|\}.
\end{equation}
On $\mathcal E_{\mathrm{LS}}$, the local coefficient errors are bounded by the corresponding edge magnitudes, and the designs and residual denominators are bounded away from degeneracy. The second event controls the combination of OLS and total-effect estimates used for structural subtraction in HOST:
\begin{equation}\label{par:uniform-reconstruction}
\mathcal E_{\mathrm{rec}}=\left\{
\max_{\substack{i\ne j\\j\notin\an(i)}}\left|
\widehat B_{i\pathto j}
-\sum_{a\in\an(j)}\widehat B^{\mathrm{LS}}_{ia}\widehat B_{a\pathto j}
-B_{ij}\right|
\le C_{\mathrm{par}}\sqrt{\frac{\Sigma_{\max}\Omega_{\max}\log p}{n}}\right\}.
\end{equation}
The pairs in \eqref{par:uniform-reconstruction} include each target and every node earlier than that target in a compatible ordering. We first bound the failure probabilities of these two events and then show how they yield recovery.

\begin{lemma}[Local-fit event]\label{aux:local-fit-event}
\normalfont
Use the true-parent fits from Section~\ref{res:setup}, and suppose
\begin{equation}\label{par:fit-conditions}
n\ge C_{\mathrm{par}}(d+\log p),\qquad
nB_{ik}^2\ge C_{\mathrm{par}}\sigma_k^2
(\Sigma_{\pa(k),\pa(k)}^{-1})_{ii}\log p
\quad(i\to k).
\end{equation}
Then $\P(\mathcal E_{\mathrm{LS}}^c)\le5p^{-6}$.
\end{lemma}

\begin{proof}
On $\mathcal E_k$, inverse monotonicity and \eqref{res:precision-bounds} give
\begin{equation}\label{res:local-energy}
\|h_{ik}\|_2^2=[(D_k^\top D_k)^{-1}]_{ii}
\le\frac4n(\Sigma_{\pa(k),\pa(k)}^{-1})_{ii}
\le\frac{4\Omega_{ii}}n
\qquad(i\in\pa(k)).
\end{equation}
The coefficient error satisfies $(\widehat B^{\mathrm{LS}}_{ik}-B_{ik})\mid D_k\sim N(0,\sigma_k^2\|h_{ik}\|_2^2)$. On $\mathcal E_k$, the Gaussian tail bound, \eqref{res:local-energy}, and the signal condition in \eqref{par:fit-conditions} give
\begin{equation}\label{par:local-coefficient-tail}
\P\{|\widehat B^{\mathrm{LS}}_{ik}-B_{ik}|>|B_{ik}|\mid D_k\}
\le2\exp\!\left\{-\frac{nB_{ik}^2}{8\sigma_k^2(\Sigma_{\pa(k),\pa(k)}^{-1})_{ii}}\right\}
\le2p^{-8}.
\end{equation}
Also $\|\widehat r_k\|_2^2/\sigma_k^2\mid D_k\sim\chi^2_{r_k}$ with $r_k=n-|\pa(k)|$. For sufficiently large $C_{\mathrm{par}}$, the sample condition gives $r_k-2\sqrt{8r_k\log p}\ge n/2$, so \eqref{aux:chi-square-lower-tail} yields
\begin{equation}\label{par:local-residual-tail}
\P\{\|\widehat r_k\|_2^2<n\sigma_k^2/2\mid D_k\}\le p^{-8}.
\end{equation}
For a nonempty parent set, put $m=|\pa(k)|$. The matrix $D_k\Sigma_{\pa(k),\pa(k)}^{-1/2}$ is standard Gaussian, and the unit sphere $\mathcal T=\{z\in\R^m:\|z\|_2=1\}$ has Gaussian width $\omega(\mathcal T)=\E\|g\|_2=a_m\le\sqrt m\le\sqrt d$, where $g\sim N(0,I_m)$. Apply Lemma~\ref{aux:gaussian-matrix-deviation} with $r=n$, $M=D_k\Sigma_{\pa(k),\pa(k)}^{-1/2}$, and $\varepsilon=(a_m+4\sqrt{\log p})/a_n$. Using $a_n^2\ge n^2/(n+1)\ge n-1$ gives
\[
\inf_{\|z\|_2=1}\|D_k\Sigma_{\pa(k),\pa(k)}^{-1/2}z\|_2
\ge\sqrt{n-1}-\sqrt d-4\sqrt{\log p}\ge\sqrt n/2
\]
with failure probability at most $2\exp\{-a_n^2(\varepsilon-a_m/a_n)^2/2\}=2p^{-8}$. Hence $\P(\mathcal E_k^c)\le2p^{-8}$; this also holds for an empty parent set. Taking expectations in the conditional bounds and summing over nodes and edges proves
\begin{equation}\label{par:local-fit-event}
\P(\mathcal E_{\mathrm{LS}}^c)
\le(3p+2|E|)p^{-8}\le5p^{-6}.
\end{equation}
\end{proof}

\begin{lemma}[Martingale bounds for reconstruction]\label{aux:reconstruction-martingales}
\normalfont
Suppose $n>d$, and fix $i\ne j$ with $j\notin\an(i)$. Put $K=\an(j)\cup\{i\}$ and define
\begin{equation}\label{res:bounded-weights}
\begin{gathered}[b]
\widetilde B^{\mathrm{LS}}_{ik}=\clip_{2|B_{ik}|}(\widehat B^{\mathrm{LS}}_{ik}),
\qquad R_k=\max\{\|\widehat r_k\|_2^2,n\sigma_k^2/2\},\\
b_k=\mathbf1\{k=i\}-B_{ik},\qquad
u_k=\frac{(\mathbf1\{k=i\}-\widetilde B^{\mathrm{LS}}_{ik})\widehat r_k}{R_k}
\quad(k\in K),\\
V_m=\sum_{\substack{k\in K\\k\le m}}u_k,
\qquad v_{ij}=V_p,\qquad M_m=\E[\mathbf X_j\mid\mathcal F_m].
\end{gathered}
\end{equation}
Here $V_m$ and $M_m$ are defined for every $m=0,\ldots,p$, whether or not $m\in K$; the empty sum gives $V_0=0$.
Then $(V_m)_{m=0}^p$ is a martingale, and for every $\mathcal F_{k-1}$-measurable vector $w\in\R^n$,
\begin{equation}\label{res:row-mgf}
\E[\exp\{w^\top u_k\}\mid\mathcal F_{k-1}]
\le\exp\!\left\{\frac{8b_k^2\|w\|_2^2}{n^2\sigma_k^2}\right\}
\qquad(k\in K).
\end{equation}
The event
\begin{equation}\label{res:maximal-event}
\mathcal G_{ij}=\left\{
\max_{0\le m\le p}\|V_m\|_2^2\le\frac{64\Omega_{ii}}n,
\quad \max_{0\le m\le p}\|M_m\|_2^2\le4n\Sigma_{jj}
\right\}
\end{equation}
satisfies $\P(\mathcal G_{ij}^c)\le2\exp\{-n/2\}$.
\end{lemma}

\begin{proof}
Conditional on $\mathcal F_{k-1}$, the Gaussian projections $P_k\mathbf W_k$ and $\widehat r_k=(I_n-P_k)\mathbf W_k$ are independent. By \eqref{res:local-coefficient-error}, the clipped coefficient depends only on the first projection, whereas $\widehat r_k/R_k$ is an odd function of the second. Hence
\begin{equation}\label{res:row-mean-zero}
\E[u_k\mid\mathcal F_{k-1}]=0\qquad(k\in K).
\end{equation}
We first verify the martingale claim. Each $V_m$ is $\mathcal F_m$-measurable and integrable: its summands are measurable functions of the errors exposed by step $m$, and the clipped coefficients and denominator floors give $\|u_k\|_2\le4|b_k|\|\widehat r_k\|_2/(n\sigma_k^2)$, which has finite expectation. If $m\in K$, then $V_m=V_{m-1}+u_m$, so \eqref{res:row-mean-zero} gives
\[
\E[V_m\mid\mathcal F_{m-1}]
=V_{m-1}+\E[u_m\mid\mathcal F_{m-1}]
=V_{m-1}.
\]
If $m\notin K$, then $V_m=V_{m-1}$ and the same conditional-expectation identity holds. Thus $(V_m)_{m=0}^p$ is a martingale. Since $v_{ij}=V_p$, repeated conditioning also gives $V_m=\E[v_{ij}\mid\mathcal F_m]$.

We next prove the conditional exponential-moment bound \eqref{res:row-mgf}. Iterating this bound will control the exponential moment of $\|v_{ij}\|_2^2$; together with the corresponding Gaussian bound for $\mathbf X_j$, this will let us apply Doob's inequality to control the maxima of $V_m$ and $M_m$ in \eqref{res:maximal-event}. For $\mathcal F_{k-1}$-measurable $w$ and a standard Gaussian variable $Z\sim N(0,1)$, we obtain
\[
\begin{aligned}
&\E[\exp\{w^\top u_k\}\mid\mathcal F_{k-1}]\\
&=\E[\cosh(w^\top u_k)\mid\mathcal F_{k-1}]
&&\quad\left(\because\;u_k\mid\mathcal F_{k-1}\overset{d}{=}-u_k\mid\mathcal F_{k-1}\right)\\
&\le\E\!\left[\cosh\!\left(\frac{2b_kw^\top\widehat r_k}{R_k}\right)\middle|\mathcal F_{k-1}\right]
&&\quad\left(\because\;|\mathbf1\{k=i\}-\widetilde B^{\mathrm{LS}}_{ik}|\le2|b_k|\right)\\
&\le\E\!\left[\cosh\!\left(\frac{4b_kw^\top\widehat r_k}{n\sigma_k^2}\right)\middle|\mathcal F_{k-1}\right]
&&\quad\left(\because\;R_k\ge n\sigma_k^2/2\right)\\
&=\E_Z\cosh\!\left(\frac{4b_k\|(I_n-P_k)w\|_2 Z}{n\sigma_k}\right)
&&\quad\left(\because\;\widehat r_k\mid\mathcal F_{k-1}\sim N(0,\sigma_k^2(I_n-P_k))\right)\\
&=\exp\!\left\{\frac{8b_k^2\|(I_n-P_k)w\|_2^2}{n^2\sigma_k^2}\right\}
&&\quad\left(\because\;\E\cosh(tZ)=e^{t^2/2}\right)\\
&\le\exp\!\left\{\frac{8b_k^2\|w\|_2^2}{n^2\sigma_k^2}\right\}
&&\quad\left(\because\;\|(I_n-P_k)w\|_2\le\|w\|_2\right).
\end{aligned}
\]
The path expansion and \eqref{res:precision-energy} yield
\begin{equation}\label{res:pair-energy}
\sum_{k\in K}\sigma_k^2B_{k\pathto j}^2=\Sigma_{jj}-\sigma_j^2,
\qquad \sum_{k\in K}\frac{b_k^2}{\sigma_k^2}\le\Omega_{ii}.
\end{equation}
Fix deterministic $w\in\R^n$. For $m\in K$, conditioning on $\mathcal F_{m-1}$ and applying \eqref{res:row-mgf} give
\[
\begin{aligned}
&\E\exp\{w^\top V_m\}\\
&=\E\!\left[\exp\{w^\top V_{m-1}\}
\E[\exp\{w^\top u_m\}\mid\mathcal F_{m-1}]\right]
&&\quad(\because\;V_m=V_{m-1}+u_m)\\
&\le\exp\!\left\{\frac{8b_m^2\|w\|_2^2}{n^2\sigma_m^2}\right\}
\E\exp\{w^\top V_{m-1}\}
&&\quad(\because\;\eqref{res:row-mgf}).
\end{aligned}
\]
Write $K=\{k_1<\cdots<k_s\}$ and set $k_0=0$. Since $V_m=V_{m-1}$ for $m\notin K$, we have $V_p=V_{k_s}$ and $V_{k_r-1}=V_{k_{r-1}}$ for $1\le r\le s$. The successive applications are
\begin{align}
&\E\exp\{w^\top v_{ij}\}\notag\\
&=\E\exp\{w^\top V_{k_s}\}
&&\quad(\because\;v_{ij}=V_p=V_{k_s})\notag\\
&\le\exp\!\left\{\frac{8b_{k_s}^2\|w\|_2^2}{n^2\sigma_{k_s}^2}\right\}
\E\exp\{w^\top V_{k_s-1}\}
&&\quad(\because\;\eqref{res:row-mgf})\notag\\
&=\exp\!\left\{\frac{8b_{k_s}^2\|w\|_2^2}{n^2\sigma_{k_s}^2}\right\}
\E\exp\{w^\top V_{k_{s-1}}\}
&&\quad(\because\;V_{k_s-1}=V_{k_{s-1}})\notag\\
&\le\left[\prod_{r=1}^s\exp\!\left\{\frac{8b_{k_r}^2\|w\|_2^2}{n^2\sigma_{k_r}^2}\right\}\right]
\E\exp\{w^\top V_{k_0}\}\notag\\
&=\prod_{k\in K}\exp\!\left\{\frac{8b_k^2\|w\|_2^2}{n^2\sigma_k^2}\right\}
&&\quad(\because\;V_{k_0}=V_0=0)\notag\\
&=\exp\!\left\{\frac{8\|w\|_2^2}{n^2}\sum_{k\in K}\frac{b_k^2}{\sigma_k^2}\right\}\notag\\
&\le\exp\!\left\{\frac{8\Omega_{ii}\|w\|_2^2}{n^2}\right\}
&&\quad\left(\because\;\sum_{k\in K}\frac{b_k^2}{\sigma_k^2}\le\Omega_{ii}\right).
\label{res:row-sum-mgf}
\end{align}

For an independent $g\sim N(0,I_n)$, Gaussian integration and \eqref{res:row-sum-mgf} give
\begin{equation}\label{res:exponential-norms}
\begin{aligned}[b]
&\E\exp\!\left\{\frac{n^2\|v_{ij}\|_2^2}{64\Omega_{ii}}\right\}\\
&=\E_g\E\exp\!\left\{\frac{ng^\top v_{ij}}{\sqrt{32\Omega_{ii}}}\right\}
&&\quad\left(\because\;\E_g\exp\{g^\top x\}=\exp\{\|x\|_2^2/2\}\right)\\
&\le\E_g\exp\{\|g\|_2^2/4\}
&&\quad\left(\because\;\eqref{res:row-sum-mgf},\ w=ng/\sqrt{32\Omega_{ii}}\right)\\
&=2^{n/2}
&&\quad\left(\because\;\E\exp\{\chi_n^2/4\}=(1-2/4)^{-n/2}\right).
\end{aligned}
\end{equation}
The expectation interchange in the first equality follows from the nonnegative integrand.
Similarly, the response has the exponential norm bound
\begin{equation}\label{res:response-exponential-moment}
\begin{aligned}[b]
\E \exp\{\|\mathbf X_j\|_2^2/(4\Sigma_{jj})\}
&=(1-2/4)^{-n/2}
&&\quad\left(\because\;\|\mathbf X_j\|_2^2/\Sigma_{jj}\sim\chi_n^2\right)\\
&=2^{n/2}.
\end{aligned}
\end{equation}

We now apply the maximal inequality directly to $V_m$ and $M_m$. For $a>0$ with $\E\exp\{a\|v_{ij}\|_2^2\}<\infty$, the identity $V_m=\E[v_{ij}\mid\mathcal F_m]$ and conditional Jensen first ensure integrability:
\[
\begin{aligned}
\E\exp\{a\|V_m\|_2^2\}
&=\E\exp\{a\|\E[v_{ij}\mid\mathcal F_m]\|_2^2\}
&&\quad(\because\;V_m=\E[v_{ij}\mid\mathcal F_m])\\
&\le\E\exp\{a\|v_{ij}\|_2^2\}
&&\quad\left(\because\;\nabla^2\exp\{a\|x\|_2^2\}\succeq0\right).
\end{aligned}
\]
The submartingale property follows from the same convexity and the martingale property of $V_m$:
\[
\begin{aligned}
&\E[\exp\{a\|V_{m+1}\|_2^2\}\mid\mathcal F_m]\\
&\ge\exp\{a\|\E[V_{m+1}\mid\mathcal F_m]\|_2^2\}
&&\quad\left(\because\;\nabla^2\exp\{a\|x\|_2^2\}\succeq0\right)\\
&=\exp\{a\|V_m\|_2^2\}
&&\quad(\because\;\E[V_{m+1}\mid\mathcal F_m]=V_m).
\end{aligned}
\]
Thus $\bigl(\exp\{a\|V_m\|_2^2\}\bigr)_{m=0}^p$ is a nonnegative submartingale. Taking $a=n^2/(64\Omega_{ii})$, whose integrability follows from \eqref{res:exponential-norms}, and applying Lemma~\ref{aux:doob-maximal} with $\lambda=\exp\{n\}$ gives
\begin{equation}\label{res:martingale-maximal}
\begin{aligned}[b]
&\P\!\left\{\max_{0\le m\le p}\|V_m\|_2^2>\frac{64\Omega_{ii}}n\right\}\\
&=\P\!\left\{\max_{0\le m\le p}\exp\!\left\{\frac{n^2\|V_m\|_2^2}{64\Omega_{ii}}\right\}>\exp\{n\}\right\}
&&\quad\left(\because\;\frac{n^2}{64\Omega_{ii}}\frac{64\Omega_{ii}}n=n\right)\\
&\le\exp\{-n\}\E\exp\!\left\{\frac{n^2\|V_p\|_2^2}{64\Omega_{ii}}\right\}
&&\quad(\because\;\text{Lemma~\ref{aux:doob-maximal}},\ \lambda=\exp\{n\})\\
&=\exp\{-n\}\E\exp\!\left\{\frac{n^2\|v_{ij}\|_2^2}{64\Omega_{ii}}\right\}
&&\quad(\because\;V_p=v_{ij})\\
&\le\exp\{-n/2\}
&&\quad(\because\;\eqref{res:exponential-norms}).
\end{aligned}
\end{equation}

For $M_m=\E[\mathbf X_j\mid\mathcal F_m]$, the tower property gives $\E[M_{m+1}\mid\mathcal F_m]=M_m$ for every $0\le m<p$, so $(M_m)_{m=0}^p$ is a martingale. Moreover, \eqref{ord:sample-error-path} and the topological labeling make $\mathbf X_j$ measurable with respect to $\mathcal F_j$. Hence $M_m=\mathbf X_j$ for $m\ge j$, in particular $M_p=\mathbf X_j$.
The same Jensen argument, now with $a=1/(4\Sigma_{jj})$ and \eqref{res:response-exponential-moment}, makes $\bigl(\exp\{\|M_m\|_2^2/(4\Sigma_{jj})\}\bigr)_{m=0}^p$ a nonnegative submartingale. Applying Lemma~\ref{aux:doob-maximal} again yields
\[
\begin{aligned}
&\P\!\left\{\max_{0\le m\le p}\|M_m\|_2^2>4n\Sigma_{jj}\right\}\\
&=\P\!\left\{\max_{0\le m\le p}\exp\!\left\{\frac{\|M_m\|_2^2}{4\Sigma_{jj}}\right\}>\exp\{n\}\right\}
&&\quad\left(\because\;\frac{4n\Sigma_{jj}}{4\Sigma_{jj}}=n\right)\\
&\le\exp\{-n\}\E\exp\!\left\{\frac{\|M_p\|_2^2}{4\Sigma_{jj}}\right\}
&&\quad(\because\;\text{Lemma~\ref{aux:doob-maximal}},\ \lambda=\exp\{n\})\\
&=\exp\{-n\}\E\exp\!\left\{\frac{\|\mathbf X_j\|_2^2}{4\Sigma_{jj}}\right\}
&&\quad(\because\;M_p=\mathbf X_j)\\
&\le\exp\{-n/2\}
&&\quad(\because\;\eqref{res:response-exponential-moment}).
\end{aligned}
\]
Finally, the definition of $\mathcal G_{ij}$ in \eqref{res:maximal-event} and a union bound give
\[
\begin{aligned}
\P(\mathcal G_{ij}^c)
&\le\P\!\left\{\max_m\|V_m\|_2^2>64\Omega_{ii}/n\right\}
+\P\!\left\{\max_m\|M_m\|_2^2>4n\Sigma_{jj}\right\}\\
&\le2\exp\{-n/2\}.
\end{aligned}
\]
\end{proof}

\begin{lemma}[Uniform reconstruction]\label{aux:reconstruction-event}
\normalfont
Use the true-parent fits from Section~\ref{res:setup}. If $n\ge C_{\mathrm{par}}(d+\log p)$, then
\[
\P(\mathcal E_{\mathrm{LS}}\cap\mathcal E_{\mathrm{rec}}^c)\le2p^{-6}.
\]
\end{lemma}

\begin{proof}
Fix $i\ne j$ with $j\notin\an(i)$ and use the notation of Lemma~\ref{aux:reconstruction-martingales}. Recall that $K=\an(j)\cup\{i\}$. Define
\begin{equation}\label{res:pair-reconstruction}
H=\widehat B_{i\pathto j}
-\sum_{a\in\an(j)}\widehat B^{\mathrm{LS}}_{ia}\widehat B_{a\pathto j}-B_{ij}.
\end{equation}
On $\mathcal E_{\mathrm{LS}}$, clipping and denominator floors are inactive:
\[
\widetilde B^{\mathrm{LS}}_{ik}=\widehat B^{\mathrm{LS}}_{ik},
\qquad R_k=\|\widehat r_k\|_2^2
\qquad(k\in K).
\]
For each node $\ell\in V=[p]$, set $\gamma_\ell=\mathbf1\{\ell=j\}+B_{\ell\pathto j}$. Then \eqref{ord:sample-error-path} gives $\mathbf X_j=\sum_{\ell\in V}\gamma_\ell\mathbf W_\ell$: the term $\ell=j$ is $\mathbf W_j$, the ancestor terms are their transmitted errors, and all other terms are zero. Substituting \eqref{res:total-effect} into the definition of $H$ gives, on $\mathcal E_{\mathrm{LS}}$,
\begin{align*}
H
&=\frac{\widehat r_i^\top\mathbf X_j}{R_i}
-\sum_{a\in\an(j)}\widetilde B^{\mathrm{LS}}_{ia}\frac{\widehat r_a^\top\mathbf X_j}{R_a}-B_{ij}\\
&=\left(\sum_{k\in K}u_k\right)^\top\mathbf X_j-B_{ij}\\
&=\sum_{k\in K}\sum_{\ell\in V}\gamma_\ell u_k^\top\mathbf W_\ell-B_{ij}\\
&=\underbrace{\sum_{k\in K}\gamma_k u_k^\top\mathbf W_k}_{\mathrm I}
+\underbrace{\sum_{k\in K}\sum_{\substack{\ell\in V\\\ell\ne k}}\gamma_\ell u_k^\top\mathbf W_\ell}_{\mathrm{II}}
-B_{ij}.
\end{align*}
For $\mathrm I$, we obtain on $\mathcal E_{\mathrm{LS}}$
\begin{equation}\label{res:reconstruction-diagonal-cancellation}
\begin{aligned}[b]
\mathrm I
&=\sum_{k\in K}\gamma_k(\mathbf1\{k=i\}-\widehat B^{\mathrm{LS}}_{ik})
\frac{\widehat r_k^\top\mathbf W_k}{R_k}
&&\quad(\because\ \widetilde B^{\mathrm{LS}}_{ik}=\widehat B^{\mathrm{LS}}_{ik})\\
&=\sum_{k\in K}\gamma_k(\mathbf1\{k=i\}-\widehat B^{\mathrm{LS}}_{ik})
&&\quad(\because\ \widehat r_k^\top\mathbf W_k=R_k)\\
&=\sum_{k\in K}B_{k\pathto j}(\mathbf1\{k=i\}-\widehat B^{\mathrm{LS}}_{ik})
&&\quad(\because\ j\notin K\Rightarrow\gamma_k=B_{k\pathto j})\\
&=\sum_{k\in K}B_{k\pathto j}(\mathbf1\{k=i\}-B_{ik}-h_{ik}^\top\mathbf W_k)
&&\quad(\because\ \eqref{res:local-coefficient-error})\\
&=B_{ij}-\sum_{k\in K}B_{k\pathto j}h_{ik}^\top\mathbf W_k
&&\quad(\because\ \eqref{eq:structural-pilot}).
\end{aligned}
\end{equation}
For $\mathrm{II}$, partition the pairs $(k,\ell)$ with $k\ne\ell$ according to which index is larger:
\[
\begin{aligned}
\mathrm{II}
&=\sum_{k\in K}\sum_{\substack{\ell\in V\\k<\ell}}
\gamma_\ell u_k^\top\mathbf W_\ell
+\sum_{k\in K}\sum_{\substack{\ell\in V\\\ell<k}}
\gamma_\ell u_k^\top\mathbf W_\ell\\
&=\sum_{\ell\in V}\gamma_\ell
\left(\sum_{\substack{k\in K\\k<\ell}}u_k\right)^\top\mathbf W_\ell
+\sum_{k\in K}u_k^\top
\left(\sum_{\substack{\ell\in V\\\ell<k}}\gamma_\ell\mathbf W_\ell\right).
\end{aligned}
\]
In the first term, fix $\ell$ and collect all earlier increments $u_k$. Their sum is
\[
\sum_{\substack{k\in K\\k<\ell}}u_k=V_{\ell-1}
\qquad\left(\because\ V_m=\sum_{\substack{k\in K\\k\le m}}u_k\right).
\]
In the second term, fix $k$ and collect the errors revealed before stage $k$. Since $\mathcal F_{k-1}=\sigma(\mathbf W_1,\ldots,\mathbf W_{k-1})$, independence and centering give
\[
\E[\mathbf W_\ell\mid\mathcal F_{k-1}]
=\begin{cases}
\mathbf W_\ell,&\ell<k,\\
0,&\ell\ge k.
\end{cases}
\]
Consequently,
\[
\begin{aligned}
\sum_{\substack{\ell\in V\\\ell<k}}\gamma_\ell\mathbf W_\ell
&=\E\left[\sum_{\ell\in V}\gamma_\ell\mathbf W_\ell\,\middle|\,\mathcal F_{k-1}\right]\\
&=\E[\mathbf X_j\mid\mathcal F_{k-1}]
&&\quad\left(\because\ \mathbf X_j=\sum_{\ell\in V}\gamma_\ell\mathbf W_\ell\right)\\
&=M_{k-1}.
\end{aligned}
\]
Substituting these two sums back into $\mathrm{II}$ gives
\[
\begin{aligned}
\mathrm{II}
&=\sum_{\ell\in V}\gamma_\ell V_{\ell-1}^\top\mathbf W_\ell
+\sum_{k\in K}u_k^\top M_{k-1}\\
&=V_{j-1}^\top\mathbf W_j
+\sum_{\ell\in\an(j)}B_{\ell\pathto j}V_{\ell-1}^\top\mathbf W_\ell
+\sum_{k\in K}u_k^\top M_{k-1}
\qquad\left(\because\ \gamma_\ell=\mathbf1\{\ell=j\}+B_{\ell\pathto j}\right)\\
&=V_{j-1}^\top\mathbf W_j
+\sum_{k\in K}B_{k\pathto j}V_{k-1}^\top\mathbf W_k
+\sum_{k\in K}u_k^\top M_{k-1}
\qquad\left(\because\ B_{k\pathto j}=0\ \text{for }k\notin\an(j)\right).
\end{aligned}
\]
Combining $H=\mathrm I+\mathrm{II}-B_{ij}$ with \eqref{res:reconstruction-diagonal-cancellation} yields, on $\mathcal E_{\mathrm{LS}}$,
\begin{equation}\label{res:direction-identity}
H=V_{j-1}^\top\mathbf W_j
+\sum_{k\in K}\left\{
B_{k\pathto j}\mathbf W_k^\top(V_{k-1}-h_{ik})+u_k^\top M_{k-1}
\right\}.
\end{equation}

We next bound the reconstruction error $H$ by an exponential-moment argument. To use conditional expectations without conditioning on the full event $\mathcal E_{\mathrm{LS}}$, put
\[
h_{ik}^{\circ}=\mathbf1_{\mathcal E_k}h_{ik}.
\]
Both $\mathcal E_k$ and $h_{ik}$ depend only on the parent design, so $h_{ik}^{\circ}$ is $\mathcal F_{k-1}$-measurable. Moreover, \eqref{res:local-energy} gives $\|h_{ik}^{\circ}\|_2\le2\sqrt{\Omega_{ii}/n}$, and $h_{ik}^{\circ}=h_{ik}$ on $\mathcal E_{\mathrm{LS}}$. Define, for $k=1,\ldots,p$,
\[
Z_k=\begin{cases}
V_{j-1}^\top\mathbf W_j,&k=j,\\
B_{k\pathto j}\mathbf W_k^\top(V_{k-1}-h_{ik}^{\circ})+u_k^\top M_{k-1},&k\in K,\\
0,&k\notin K\cup\{j\}.
\end{cases}
\]
These cases are disjoint because $j\notin K$. For $k\in K$, $\mathcal F_{k-1}$-measurability of $V_{k-1},h_{ik}^{\circ},M_{k-1}$ gives
\[
\begin{aligned}
\E[Z_k\mid\mathcal F_{k-1}]
&=B_{k\pathto j}(V_{k-1}-h_{ik}^{\circ})^\top\E[\mathbf W_k\mid\mathcal F_{k-1}]
+M_{k-1}^\top\E[u_k\mid\mathcal F_{k-1}]\\
&=0.
\end{aligned}
\]
The same conditional mean is zero for $k=j$, since $\E[\mathbf W_j\mid\mathcal F_{j-1}]=0$. Thus $(\sum_{k=1}^m Z_k)_{m=0}^p$ is a martingale. On $\mathcal E_{\mathrm{LS}}$, \eqref{res:direction-identity} gives $\sum_{k=1}^p Z_k=H$.

For $t\in\R$, the two terms of $Z_k$ have the following conditional exponential bounds:
\[
\begin{aligned}
&\E\!\left[\exp\{2tB_{k\pathto j}\mathbf W_k^\top(V_{k-1}-h_{ik}^{\circ})\}\mid\mathcal F_{k-1}\right]\\
&\quad=\exp\{2t^2\sigma_k^2B_{k\pathto j}^2\|V_{k-1}-h_{ik}^{\circ}\|_2^2\}
\qquad(\because\ \mathbf W_k\mid\mathcal F_{k-1}\sim N(0,\sigma_k^2I_n)),\\
&\E[\exp\{2tu_k^\top M_{k-1}\}\mid\mathcal F_{k-1}]\\
&\quad\le\exp\!\left\{\frac{32t^2b_k^2\|M_{k-1}\|_2^2}{n^2\sigma_k^2}\right\}
\qquad(\because\ \eqref{res:row-mgf},\ w=2tM_{k-1}).
\end{aligned}
\]
Conditional Cauchy--Schwarz therefore yields, for $k\in K$,
\[
\begin{aligned}
&\E[\exp\{tZ_k\}\mid\mathcal F_{k-1}]\\
&\le\Bigl(\E[\exp\{2tB_{k\pathto j}\mathbf W_k^\top(V_{k-1}-h_{ik}^{\circ})\}\mid\mathcal F_{k-1}]\Bigr)^{1/2}
\Bigl(\E[\exp\{2tu_k^\top M_{k-1}\}\mid\mathcal F_{k-1}]\Bigr)^{1/2}\\
&\le\exp\!\left\{t^2\left(\sigma_k^2B_{k\pathto j}^2\|V_{k-1}-h_{ik}^{\circ}\|_2^2
+\frac{16b_k^2\|M_{k-1}\|_2^2}{n^2\sigma_k^2}\right)\right\}.
\end{aligned}
\]
Accordingly, define the $\mathcal F_{k-1}$-measurable quantities
\[
v_k=\begin{cases}
\sigma_j^2\|V_{j-1}\|_2^2,&k=j,\\
\sigma_k^2B_{k\pathto j}^2\|V_{k-1}-h_{ik}^{\circ}\|_2^2
+\dfrac{16b_k^2\|M_{k-1}\|_2^2}{n^2\sigma_k^2},&k\in K,\\
0,&k\notin K\cup\{j\},
\end{cases}
\qquad Q=\sum_{k=1}^p v_k.
\]
For each fixed $t\in\R$, we will show that
\[
\left(\exp\!\left\{t\sum_{k=1}^m Z_k-t^2\sum_{k=1}^m v_k\right\}\right)_{m=0}^p
\]
is a nonnegative supermartingale with respect to $(\mathcal F_m)_{m=0}^p$, starting at one. This will bound its terminal expectation by one.

At node $j$, $V_{j-1}$ is $\mathcal F_{j-1}$-measurable and $\mathbf W_j\sim N(0,\sigma_j^2I_n)$ is independent of $\mathcal F_{j-1}$. Hence
\[
Z_j\mid\mathcal F_{j-1}
\sim N\!\left(0,\sigma_j^2\|V_{j-1}\|_2^2\right)
=N(0,v_j).
\]
The Gaussian moment-generating function therefore gives
\[
\E[\exp\{tZ_j\}\mid\mathcal F_{j-1}]
=\exp\{t^2v_j/2\}\le\exp\{t^2v_j\}.
\]
Thus $\E[\exp\{tZ_k-t^2v_k\}\mid\mathcal F_{k-1}]\le1$ for every $k$. For $m=1,\ldots,p$, it follows that
\[
\begin{aligned}
&\E\!\left[\exp\!\left\{t\sum_{k=1}^mZ_k-t^2\sum_{k=1}^mv_k\right\}\middle|\mathcal F_{m-1}\right]\\
&=\exp\!\left\{t\sum_{k=1}^{m-1}Z_k-t^2\sum_{k=1}^{m-1}v_k\right\}
\E[\exp\{tZ_m-t^2v_m\}\mid\mathcal F_{m-1}]\\
&\le\exp\!\left\{t\sum_{k=1}^{m-1}Z_k-t^2\sum_{k=1}^{m-1}v_k\right\}.
\end{aligned}
\]
This proves the nonnegative supermartingale property. Taking expectations in the preceding inequality removes the last increment from both sums. Repeating this step for $m=p,p-1,\ldots,1$, with both sums empty at $m=0$, gives
\begin{equation}\label{res:scalar-mgf}
\begin{aligned}[b]
\E[\mathbf1_{\mathcal E_{\mathrm{LS}}}\exp\{tH-t^2Q\}]
&=\E\!\left[\mathbf1_{\mathcal E_{\mathrm{LS}}}\exp\!\left\{t\sum_{k=1}^pZ_k-t^2Q\right\}\right]\\
&\le\E\exp\!\left\{t\sum_{k=1}^pZ_k-t^2\sum_{k=1}^pv_k\right\}\\
&\le\E\exp\!\left\{t\sum_{k=1}^{p-1}Z_k-t^2\sum_{k=1}^{p-1}v_k\right\}\\
&\le\E\exp\!\left\{t\sum_{k=1}^{0}Z_k-t^2\sum_{k=1}^{0}v_k\right\}\\
&=1\qquad(t\in\R).
\end{aligned}
\end{equation}

We now bound $Q$ on $\mathcal G_{ij}$. In particular, \eqref{res:maximal-event} gives
\[
\|V_{j-1}\|_2^2\le\max_{0\le m\le p}\|V_m\|_2^2\le\frac{64\Omega_{ii}}n.
\]
By \eqref{res:maximal-event} and \eqref{res:local-energy},
\[
\|V_{k-1}-h_{ik}^{\circ}\|_2
\le\|V_{k-1}\|_2+\|h_{ik}^{\circ}\|_2
\le8\sqrt{\Omega_{ii}/n}+2\sqrt{\Omega_{ii}/n}
=10\sqrt{\Omega_{ii}/n}.
\]
Using also $\|M_{k-1}\|_2^2\le4n\Sigma_{jj}$ from \eqref{res:maximal-event}, we obtain
\begin{equation}\label{res:variance-bound}
\begin{aligned}[b]
Q=\sum_{k=1}^p v_k
&=\sigma_j^2\|V_{j-1}\|_2^2
+\sum_{k\in K}\left\{\sigma_k^2B_{k\pathto j}^2\|V_{k-1}-h_{ik}^{\circ}\|_2^2
+\frac{16b_k^2\|M_{k-1}\|_2^2}{n^2\sigma_k^2}\right\}\\
&\le\frac{64\Omega_{ii}\sigma_j^2}{n}
+\frac{100\Omega_{ii}}{n}\sum_{k\in K}\sigma_k^2B_{k\pathto j}^2
+\frac{64\Sigma_{jj}}{n}\sum_{k\in K}\frac{b_k^2}{\sigma_k^2}\\
&\le\frac{\Omega_{ii}}n\{64\sigma_j^2+100(\Sigma_{jj}-\sigma_j^2)+64\Sigma_{jj}\}
\qquad(\because\ \eqref{res:pair-energy})\\
&=\frac{\Omega_{ii}}n(164\Sigma_{jj}-36\sigma_j^2)\\
&\le\frac{164\Omega_{ii}\Sigma_{jj}}n.
\end{aligned}
\end{equation}
For $u>0$ and $t>0$, on $\mathcal G_{ij}\cap\{H>u\}$ we have
\[
tH-t^2Q>tu-\frac{164t^2\Omega_{ii}\Sigma_{jj}}n.
\]
Exponential Markov and \eqref{res:scalar-mgf} therefore give
\[
\begin{aligned}
&\P(\mathcal E_{\mathrm{LS}}\cap\mathcal G_{ij}\cap\{H>u\})\\
&\le\P\!\left\{\mathbf1_{\mathcal E_{\mathrm{LS}}}\exp\{tH-t^2Q\}
>\exp\!\left\{tu-\frac{164t^2\Omega_{ii}\Sigma_{jj}}n\right\}\right\}\\
&\le\exp\!\left\{-tu+\frac{164t^2\Omega_{ii}\Sigma_{jj}}n\right\}
\E[\mathbf1_{\mathcal E_{\mathrm{LS}}}\exp\{tH-t^2Q\}]\\
&\le\exp\!\left\{-tu+\frac{164t^2\Omega_{ii}\Sigma_{jj}}n\right\}
\qquad(\because\ \eqref{res:scalar-mgf})\\
&=\exp\!\left\{-\frac{nu^2}{656\Omega_{ii}\Sigma_{jj}}\right\}
\qquad\left(\text{substituting }t=\frac{nu}{328\Omega_{ii}\Sigma_{jj}}\right).
\end{aligned}
\]
Applying the same argument to $-H$ uses \eqref{res:scalar-mgf} with $-t$ and gives the same bound. Adding the two tails and the probability of $\mathcal G_{ij}^c$ yields
\begin{equation}\label{res:scalar-tail}
\begin{aligned}[b]
\P(\mathcal E_{\mathrm{LS}}\cap\{|H|>u\})
&\le\P(\mathcal G_{ij}^c)
+\P(\mathcal E_{\mathrm{LS}}\cap\mathcal G_{ij}\cap\{|H|>u\})\\
&\le2\exp\{-n/2\}+2\exp\!\left\{-\frac{nu^2}{656\Omega_{ii}\Sigma_{jj}}\right\}.
\end{aligned}
\end{equation}
Finally, take $u=C_{\mathrm{par}}\sqrt{\Omega_{ii}\Sigma_{jj}\log p/n}$. Enlarge $C_{\mathrm{par}}$ so that $n\ge18\log p$ and $C_{\mathrm{par}}^2\ge5904=9\cdot656$. For $p\ge2$, the two terms in \eqref{res:scalar-tail} satisfy
\[
2\exp\{-n/2\}+2\exp\{-C_{\mathrm{par}}^2\log p/656\}
\le4p^{-9}\le2p^{-8}.
\]
The preceding bound holds for each fixed pair $i\ne j$ with $j\notin\an(i)$. Write $H_{ij}$ for the error $H$ in \eqref{res:pair-reconstruction} to make its dependence on the pair explicit. By \eqref{par:uniform-reconstruction}, $\mathcal E_{\mathrm{rec}}$ requires every such error to be below the common threshold. Its complement is therefore
\[
\mathcal E_{\mathrm{rec}}^c
=\bigcup_{\substack{i\ne j\\j\notin\an(i)}}
\left\{|H_{ij}|>C_{\mathrm{par}}\sqrt{\frac{\Sigma_{\max}\Omega_{\max}\log p}{n}}\right\}.
\]
For each pair, this threshold is at least the value of $u$ used above, so
\[
\begin{aligned}
&\P\!\left(\mathcal E_{\mathrm{LS}}\cap
\left\{|H_{ij}|>C_{\mathrm{par}}\sqrt{\frac{\Sigma_{\max}\Omega_{\max}\log p}{n}}\right\}\right)\\
&\le\P\!\left(\mathcal E_{\mathrm{LS}}\cap
\left\{|H_{ij}|>C_{\mathrm{par}}\sqrt{\frac{\Sigma_{jj}\Omega_{ii}\log p}{n}}\right\}\right)
\qquad(\because\ \Sigma_{jj}\Omega_{ii}\le\Sigma_{\max}\Omega_{\max})\\
&\le2p^{-8}.
\end{aligned}
\]
A union bound over at most $p^2$ possible pairs $(i,j)$ now gives
\begin{equation}\label{par:reconstruction-event-probability}
\P(\mathcal E_{\mathrm{LS}}\cap\mathcal E_{\mathrm{rec}}^c)
\le\sum_{\substack{i\ne j\\j\notin\an(i)}}2p^{-8}
\le p^2\cdot2p^{-8}=2p^{-6}.
\end{equation}
\end{proof}

\subsubsection{Completion of the proof}
We first obtain a single event on which all oracle reconstructions are accurate. We then show that, on this event, the algorithm uses exactly these reconstructions and makes every parent decision correctly.

Since $\sigma_k^2\le\Sigma_{kk}\le\Sigma_{\max}$ and $(\Sigma_{\pa(k),\pa(k)}^{-1})_{ii}\le\Omega_{ii}\le\Omega_{\max}$, the sample and signal conditions in \eqref{eq:parent-rate} imply \eqref{par:fit-conditions}. Splitting according to whether $\mathcal E_{\mathrm{LS}}$ holds, Lemmas~\ref{aux:local-fit-event} and~\ref{aux:reconstruction-event} give
\begin{equation}\label{par:common-event-probability}
\P(\mathcal E_{\mathrm{rec}}^c)
\le\P(\mathcal E_{\mathrm{LS}}^c)+\P(\mathcal E_{\mathrm{LS}}\cap\mathcal E_{\mathrm{rec}}^c)
\le5p^{-6}+2p^{-6}=7p^{-6}\le p^{-2}.
\end{equation}
Choose the threshold multiplier and then the signal factor in \eqref{eq:parent-rate} so that
\begin{equation}\label{par:threshold-choice}
C_{\mathrm{par}}\sqrt{\frac{\Sigma_{\max}\Omega_{\max}\log p}{n}}
\le(1-\delta)\min_j\tau_j,
\qquad 2\max_j\tau_j<B_{\min}.
\end{equation}

Fix a sample in $\mathcal E_{\mathrm{rec}}$ and any ordering $\widehat\pi$ compatible with $G$. We use induction over the targets $j=\widehat\pi_r$. Compatibility implies $\pa(\widehat\pi_1)=\emptyset$, so the initialization $\widehat T_{\widehat\pi_1}=\emptyset$ is correct.

Suppose now that every target $a$ preceding $j$ has been recovered correctly. Then
\[
\widehat T_a=\pa(a)
\quad\Longrightarrow\quad
\mathbf X_{\widehat T_a}^{\dagger}\mathbf X_a
=\mathbf X_{\pa(a)}^{\dagger}\mathbf X_a.
\]
Thus the OLS coefficients and residuals in \eqref{eq:local-ols} coincide with the true-parent fits in \eqref{res:local-fits} used to define $\mathcal E_{\mathrm{rec}}$.

For this target $j$, the algorithm processes earlier nodes in reverse order. After $m$ steps, we claim that
\[
\begin{aligned}
\widehat T_j^{(m)}
&=\pa(j)\cap\{\widehat\pi_{r-m},\ldots,\widehat\pi_{r-1}\},\\
\widehat Q_j^{(m)}
&=\an(j)\cap\{\widehat\pi_{r-m},\ldots,\widehat\pi_{r-1}\}.
\end{aligned}
\]
For $m=0$, the processed set is empty and both identities hold by initialization. Suppose they hold after $m-1$ steps, and let $k=\widehat\pi_{r-m}$ be the next earlier node to be processed.

We first show that replacing $\widehat Q_j^{(m-1)}$ by $\an(j)$ does not change the sum in \eqref{eq:structural-pilot}. By the induction hypothesis,
\[
\widehat Q_j^{(m-1)}
=\an(j)\cap\{\widehat\pi_{r-m+1},\ldots,\widehat\pi_{r-1}\}
\subseteq\an(j).
\]
Thus every node included in the reconstruction sum is an ancestor of $j$. Any ancestor missing from this set has not yet been processed. Since the current node is $k=\widehat\pi_{r-m}$, this gives
\[
\begin{aligned}
a\in\an(j)\setminus\widehat Q_j^{(m-1)}
&\ \Longrightarrow\ \widehat\pi^{-1}(a)\le r-m=\widehat\pi^{-1}(k)\\
&\ \Longrightarrow\ k\notin\pa(a)
\qquad(\because\ k\in\pa(a)\Rightarrow\widehat\pi^{-1}(k)<\widehat\pi^{-1}(a))\\
&\ \Longrightarrow\ \widehat B^{\mathrm{LS}}_{ka}=0
\qquad(\because\ \eqref{res:local-fits}).
\end{aligned}
\]
Consequently, every additional term obtained by extending the sum to all ancestors is zero:
\[
\begin{aligned}
\sum_{a\in\an(j)}\widehat B^{\mathrm{LS}}_{ka}\widehat B_{a\pathto j}
&=\sum_{a\in\widehat Q_j^{(m-1)}}\widehat B^{\mathrm{LS}}_{ka}\widehat B_{a\pathto j}
+\sum_{a\in\an(j)\setminus\widehat Q_j^{(m-1)}}\widehat B^{\mathrm{LS}}_{ka}\widehat B_{a\pathto j}\\
&=\sum_{a\in\widehat Q_j^{(m-1)}}\widehat B^{\mathrm{LS}}_{ka}\widehat B_{a\pathto j}.
\end{aligned}
\]
Substituting this equality into \eqref{eq:structural-pilot} gives
\begin{equation}\label{par:algorithm-reconstruction}
\widehat B^{\mathrm{res}}_{kj}
=\widehat B_{k\pathto j}
-\sum_{a\in\widehat Q_j^{(m-1)}}\widehat B^{\mathrm{LS}}_{ka}\widehat B_{a\pathto j}
=\widehat B_{k\pathto j}
-\sum_{a\in\an(j)}\widehat B^{\mathrm{LS}}_{ka}\widehat B_{a\pathto j}.
\end{equation}
Since $k$ precedes $j$ in a compatible ordering, $j\notin\an(k)$. This pair is therefore covered by \eqref{par:uniform-reconstruction}. The first inequality in \eqref{par:threshold-choice} leaves room for the correction of size at most $\delta\tau_j$ in \eqref{eq:corrected-statistic}. Combining the reconstruction error bound with this correction gives
\begin{equation}\label{eq:correction-envelope}
\begin{aligned}[b]
|\widetilde B_{kj}-B_{kj}|
&\le|\widehat B^{\mathrm{res}}_{kj}-B_{kj}|+\delta\tau_j
&&\quad(\because\ \eqref{eq:corrected-statistic})\\
&\le C_{\mathrm{par}}\sqrt{\frac{\Sigma_{\max}\Omega_{\max}\log p}{n}}+\delta\tau_j
&&\quad(\because\ \mathcal E_{\mathrm{rec}})\\
&\le(1-\delta)\tau_j+\delta\tau_j
&&\quad(\because\ \eqref{par:threshold-choice})\\
&=\tau_j.
\end{aligned}
\end{equation}
The second inequality in \eqref{par:threshold-choice}, $B_{\min}>2\max_j\tau_j$, now separates nonzero coefficients from zero. Together with \eqref{eq:correction-envelope}, it gives both directions of the threshold decision:
\[
\begin{aligned}
B_{kj}=0&\ \Longrightarrow\ |\widetilde B_{kj}|\le\tau_j,\\
B_{kj}\ne0&\ \Longrightarrow\ |\widetilde B_{kj}|\ge|B_{kj}|-\tau_j
\ge B_{\min}-\tau_j>\tau_j.
\end{aligned}
\]
Hence $k$ is selected exactly when it is a parent of $j$, preserving the claimed identity for $\widehat T_j^{(m)}$.

It remains to verify the induction identity for $\widehat Q_j^{(m)}$. Only the current node $k=\widehat\pi_{r-m}$ can be added at this step, so it suffices to show that the update adds $k$ exactly when $k\in\an(j)$.

If $k\in\pa(j)$, the correctness of the parent decision gives $k\in\widehat T_j^{(m)}$, and \eqref{eq:reverse-set} adds $k$. If $k\in\an(j)\setminus\pa(j)$, choose a directed path $k\to a\to\cdots\to j$ of length at least two. Compatibility of the ordering gives
\[
r-m=\widehat\pi^{-1}(k)<\widehat\pi^{-1}(a)<\widehat\pi^{-1}(j)=r.
\]
Thus $a$ is one of the earlier nodes already processed in reverse order. Since $a\in\an(j)$, the induction hypothesis implies
\[
a\in\an(j)\cap\{\widehat\pi_{r-m+1},\ldots,\widehat\pi_{r-1}\}
=\widehat Q_j^{(m-1)}.
\]
Moreover, $k\to a$ means $k\in\pa(a)=\widehat T_a$, where the equality follows from correctness of the earlier target $a$. Hence the second condition in \eqref{eq:reverse-set} also adds $k$.

Conversely, if the update adds $k$, then either $k\in\widehat T_j^{(m)}$ or $k\in\widehat T_a$ for some $a\in\widehat Q_j^{(m-1)}$. The first case implies $k\in\pa(j)\subseteq\an(j)$ by the correctness of the parent decision. In the second case, the induction hypotheses give $k\in\pa(a)$ and $a\in\an(j)$, so concatenating $k\to a$ with a path from $a$ to $j$ gives $k\in\an(j)$. We have therefore established
\begin{equation}\label{par:ancestor-update}
k\in\an(j)\quad\Longleftrightarrow\quad
k\in\widehat T_j^{(m)}\ \text{or}\
k\in\bigcup_{a\in\widehat Q_j^{(m-1)}}\widehat T_a.
\end{equation}
Equations~\eqref{eq:reverse-set} and~\eqref{par:ancestor-update} give
\[
\widehat Q_j^{(m)}
=\widehat Q_j^{(m-1)}\cup\bigl(\{k\}\cap\an(j)\bigr)
=\an(j)\cap\{\widehat\pi_{r-m},\ldots,\widehat\pi_{r-1}\},
\]
completing the induction over $m$. At $m=r-1$, the parent-set identity gives $\widehat T_j=\pa(j)$, so the OLS fit for $j$ uses its true parents. Induction over $r=1,\ldots,p$ therefore recovers all parent sets and hence $G$.

The event $\mathcal E_{\mathrm{rec}}$ covers all pairs in \eqref{par:uniform-reconstruction} and does not depend on the chosen compatible ordering. The preceding deterministic argument therefore also applies to a data-dependent ordering whenever it is compatible. In particular,
\[
\{\widehat G\ne G\}
\subseteq\{\widehat\pi\text{ is incompatible with }G\}\cup\mathcal E_{\mathrm{rec}}^c,
\]
and hence
\[
\begin{aligned}
\P\{\widehat G\ne G\}
&\le\P\{\widehat\pi\text{ is incompatible with }G\}+\P(\mathcal E_{\mathrm{rec}}^c)\\
&\le\P\{\widehat\pi\text{ is incompatible with }G\}+p^{-2}
\qquad(\because\ \eqref{par:common-event-probability}).
\end{aligned}
\]

\subsection{Proof of Theorem~\ref{thm:dag-recovery}}
\label{app:dag-recovery-proof}
\begin{proof}
The first two terms in \eqref{eq:dag-rate} give \eqref{eq:ordering-rate}. They also imply $n\gtrsim d+\log p$, because $\Sigma_{\max}/\lambda_{\min}(\Sigma)\ge1$ and $p\ge2$. For a nonempty graph, the last term of \eqref{eq:dag-rate} is equivalent, up to a fixed factor, to
\[
B_{\min}^2\gtrsim\frac{\Sigma_{\max}\Omega_{\max}\log p}{n},
\]
which is the signal condition in \eqref{eq:parent-rate}. For an empty graph this condition is vacuous. Theorems~\ref{thm:ordering-recovery} and~\ref{thm:parent-recovery}, with their stated penalties and thresholds, therefore give
\[
\P\{\widehat G\ne G\}
\le\P\{\widehat\pi\text{ is incompatible with }G\}+p^{-2}
\le2p^{-2}\le p^{-1}.
\]
\end{proof}

\subsection{Proof of Proposition~\ref{prop:computational-complexity}}
\label{app:computation-proof}

\begin{proof}
Fix the ordering penalty and parent thresholds. We use dense cyclic coordinate descent for Lasso \citep{friedman2010regularization,friedman2007pathwise} and singular value decomposition (SVD) for local OLS, suppressing solver iteration counts and numerical tolerance factors as in Section~\ref{SecMain}. First form the full-sample and two fitting-fold cross-product matrices in $O(np^2)$ work.

\paragraph{Ordering fits and coefficient updates.}
At a prefix of size $s$, a coordinate update uses the stored cross-products and costs $O(s)$. A sweep over the $s$ coordinates costs $O(s^2)$, giving the same cost per fit under the stated convention. Evaluating a fit on the other fold also costs $O(s^2)$, since the residual sum of squares is a quadratic form computed from the held-out cross-products. There are $p-s$ candidate nodes and two fitting folds, so summing over prefixes gives
\[
O\!\left(\sum_{s=1}^{p-1}2(p-s)s^2\right)=O(p^4).
\]
For the selected node, the coordinate refits use the same stored cross-products and the original averaged coefficients. Together they require $O(s^2)$ work, or $O(p^3)$ over all nodes. Scanning the scores and resolving ties adds $O(p^2)$.

\paragraph{Local regressions and total effects.}
A local design with $s\le p-1$ columns has a minimum-norm OLS solution computable by SVD in $O(ns\min\{n,s\})$ work \citep{golub2013matrix}. Forming the residual adds $O(ns+n)$, also accounting for empty designs. Each node requires only one retained local fit, so the total is at most $O(np^3)$. For node $k$ with a nonzero residual, the row vector $\widehat r_k^\top\mathbf X/\|\widehat r_k\|_2^2$ supplies its total-effect estimates to all other nodes. This multiplication costs $O(np)$, hence $O(np^2)$ over all nodes. A zero residual gives a zero row of estimates.

\paragraph{Processing earlier nodes in reverse order.}
For each target $j$, store the current reconstruction for every unprocessed earlier node and a flag recording whether it is a parent of a retained node. Start with the total-effect estimates as the reconstructions and set all flags to zero. Processing an earlier node $k$ takes constant work for clipping and thresholding. If $k$ is retained, subtract $\widehat B^{\mathrm{LS}}_{uk}\widehat B_{k\pathto j}$ from the reconstruction for each unprocessed $u\in\widehat T_k$ and set its flag. Local coefficients outside $\widehat T_k$ are zero, so these updates implement \eqref{eq:structural-pilot} and \eqref{eq:reverse-set}. With $s$ earlier nodes, at most $s$ retained nodes each update at most $s$ entries. Thus the cost is $O(s^2)$ per target and $O(p^3)$ over all target nodes.

Combining the stages gives $O(np^3+p^4)$ under the stated convention.
\end{proof}
\section{Practical Data-Driven Tuning}
\label{app:ebic-proof}
HOST-CV/EBIC selects a common ordering penalty by pooled nodewise two-fold cross-validation (CV), then selects a parent threshold for each node by the extended Bayesian information criterion (EBIC). The candidate grids $\mathcal C_{\mathrm{ord}}$ and $\mathcal C_{\mathrm{par}}$ are given in Appendices~\ref{app:simulation-tuning} and~\ref{app:realdata}.

\subsection{Ordering penalty selection}
\label{app:cv-rule}
For each $c\in\mathcal C_{\mathrm{ord}}$, use the same penalty $\lambda_{\mathrm{ord}}(c)=c\sqrt{\log p/\lfloor n/2\rfloor}$ in both folds. Let $\mathbf Z^{(\ell)}$ denote the column-centered data in fold $I_\ell$, and write $\bar\ell$ for the other fold index. Fit a Lasso regression for each node using all other nodes in that fold.

Let $\operatorname{RSS}_j^{(\ell)}(c)$ denote the residual sum of squares on the other fold $I_{\bar\ell}$, evaluated using the coefficients and column means estimated from $I_\ell$. Normalize each held-out RSS by the centered response sum of squares, then average over nodes and both folds to obtain
\[
\operatorname{CV}^{(\ell)}(c)
=\frac1p\sum_{j=1}^{p}
\frac{\operatorname{RSS}_j^{(\ell)}(c)}
{\|\mathbf Z_j^{(\bar\ell)}\|_2^2},
\qquad
\operatorname{CV}(c)
=\frac{\operatorname{CV}^{(1)}(c)+\operatorname{CV}^{(2)}(c)}2.
\]
Select the multiplier minimizing the pooled score to obtain one common ordering penalty.

\subsection{Parent threshold selection}
\label{sec:practical}
Recover parent sets sequentially along the estimated ordering. For node $j$, let $\nu_j$ be the number of earlier nodes. All full-sample covariance and EBIC calculations use the sample size $n$. Each multiplier $c\in\mathcal C_{\mathrm{par}}$ gives the threshold $\tau_j(c)=c\sqrt{\log p/n}$.
Apply the reverse processing step of Algorithm~\ref{alg1} at this threshold to obtain a candidate parent set $\widehat T_j(c)$. The fitted coefficients and residuals of earlier nodes are shared across candidates, while the retained sets $\widehat Q_j^{(m)}$ are recomputed for each threshold. When $\nu_j=0$, the candidate parent set is empty.

Evaluate each distinct candidate set $T$ by full-sample OLS after centering the columns of $\mathbf X$, with residual sum of squares
\[
\operatorname{RSS}_j(T)
=\min_{b\in\R^{|T|}}\|\mathbf X_j-\mathbf X_Tb\|_2^2.
\]
The implementation uses the nodewise score
\begin{equation}\label{eq:ebic}
\operatorname{EBIC}_j(T)
=n\log\{\operatorname{RSS}_j(T)/n\}
+|T|\log n+4|T|\log\max\{\nu_j,1\}.
\end{equation}
The first two terms measure regression fit and model size. The final term adapts the graphical EBIC penalty with $\gamma=1$ \citep{foygel2010extended} to the $\nu_j$ earlier nodes, following the EBIC construction of \citet{chen2008extended}. EBIC is evaluated only for candidate sets $T$ with a full-column-rank centered design, $\rank(\mathbf X_T)=|T|$, positive residual degrees of freedom, $n-|T|-1>0$, and a positive residual sum of squares, $\operatorname{RSS}_j(T)>0$. Otherwise, the EBIC score is set to $+\infty$.

For each node $j$, select the multiplier $\widehat c_j$ with the smallest finite EBIC score, breaking ties by fewer parents and then a larger multiplier.

\section{Simulation Protocol and Results}\label{app:simulation}
This section gives the data-generation procedures, algorithm settings, and evaluation rules for the main experiments, followed by comparisons across graph families and separate evaluations of ordering and parent recovery. The score diagnostic in Figure~\ref{graph:layered} is described separately below.

\subsection{Score diagnostic in Figure~\ref{graph:layered}}\label{app:score-diagnostic}
The graph has layer widths $64,32,16,8,4,2,1$, with each pair of nodes feeding one node in the next layer. All nodes have independent Gaussian errors with variance $0.75$. The direct effects across the six successive layer transitions are $(0.4,0.44,0.48,0.52,0.56,0.6)$, with all effects positive. The designated remaining source is the first node in the second layer, and the remaining nonsource is the single node in the final layer. Their regressions on the 64 first-layer nodes have two and 64 nonzero coefficients, respectively. Each regression coefficient for the node in the final layer is the product of the six direct effects, approximately $0.015$. The population ordering scores are $0.75$ for the designated node in the second layer and approximately $1.929$ for the node in the final layer.

We use 1,000 repetitions for each $n$ = 20, 30, 40, 50, fitting zero-intercept Lasso to raw observations with standardized predictors and fixed $\lambda=0.1$. In-sample fitting and scoring use all observations. One-way hold-out fitting and scoring use a random equal split. Scores are mean squared residuals. We summarize the scores by means plus or minus one sample standard deviation across repetitions, computed before the logarithmic axis transformation, and estimate the pairwise ordering probability with 95\% Wilson intervals.

\subsection{Data generation}
The graph-recovery and runtime experiments in Figures~\ref{fig:host-scaling}--\ref{fig:all} use the Gaussian linear SEM in \eqref{eq:linear-sem} and a common random DAG generator. Start with an auxiliary ordering $1,\ldots,p$. For each node $j$, draw its parent count uniformly from $\{0,\ldots,\min(d,j-1)\}$, then sample that many parents uniformly without replacement from $\{1,\ldots,j-1\}$. The parameter $d$ specifies the maximum indegree allowed by the generator. Nonzero direct effects have independent magnitudes uniform on $[0.4,0.6]$ and independent symmetric signs. Structural errors are independent Gaussian variables with variance $0.75$. An independent uniform permutation of the node labels shuffles the variable ordering. The additional graph families in Appendix~\ref{app:additional-graphs} use the same coefficient and error distributions with different graph-generation procedures. Each generated dataset is column-centered once before fitting.

\subsection{Graph Recovery Experiments}\label{app:simulation-tuning}
Figure~\ref{fig:host-scaling} evaluates HOST over all combinations of $p$ = 30, 50, 100, 200, $d$ = 2, 3, 4, 5, and $n$ = 128, 256, 512, 1024, 2048, with 200 repetitions per setting. It reports graph recovery rate, the fraction of exactly recovered graphs.

Figure~\ref{fig:all}(a)--(b) compares methods at $p$ = 50, 200, $d$ = 3, and $n$ = 100, 200, 300, 400, 500, with 50 paired repetitions per setting. Methods receive the same dataset in each repetition. Daskalakis Alg.~1 is evaluated only at $p$ = 50 due to the memory required to cache intermediate regression coefficients at larger dimensions. All graph-recovery experiments use the CPU implementations and tuning settings below. The implementations used for runtime measurement are specified in Appendix~\ref{app:timing-protocol}.

\paragraph{HOST and HOST-CV/EBIC.}
Both variants use our Cython/C++ implementation with $\delta=0.8$ and the same two-fold split for each dataset. HOST uses $\lambda_{\mathrm{ord}}=0.5\sqrt{\log p/\lfloor n/2\rfloor}$ and $\tau_j=c_{\mathrm{par}}\sqrt{\log p/n}$. The parent-threshold multiplier is $c_{\mathrm{par}}=3$ in Figure~\ref{fig:host-scaling} and $c_{\mathrm{par}}=2$ in the recovery and runtime comparisons, including the additional graph families. HOST-CV/EBIC selects both tuning parameters using the rules in Appendix~\ref{app:ebic-proof}. The pooled two-fold CV ordering grid is $\mathcal C_{\mathrm{ord}}=\{10^{-1+0.1k}:k=0,\ldots,20\}$. Parent recovery uses the full sample and the grid $\mathcal C_{\mathrm{par}}=\{\tfrac14 32^{k/32}:k=0,\ldots,32\}$, which contains 33 multipliers from $0.25$ to $8$.

\paragraph{HLSM.}
Our Cython/C++ implementation of \citet{park2021learning2} constructs the ordering backwards using OLS residual variances on Lasso-selected supports and selects parents from these supports, with standardized predictors and penalty $\lambda=3\sqrt{\log p/n}$.

\paragraph{TD-BSS.}
Our Cython/C++ implementation of the subset-based procedure of \citet{gao2022optimal} uses the supplied indegree bound $q=d$ in ordering, with $d$ given by an oracle. Parent selection compares the residual-variance increase on removing a candidate parent against $\eta=3\log p/n$.

\paragraph{TD-BIC.}
Our Cython/C++ implementation of TD-BIC proposed by \citet{gao2022optimal} uses the same ordering search and supplied indegree bound $q=d$ as TD-BSS. For each node, it selects a subset of at most $q$ earlier nodes minimizing $n\log(\mathrm{RSS}/n)+|S|\log n$, including the empty subset.

\paragraph{Daskalakis Alg.~1.}
Our C++ implementation follows the published ordering and parent-selection procedure \citep{daskalakis2025condition}. It uses the supplied indegree bound $q=d$ given by an oracle and coefficient threshold $b_{\min}/2=0.2$, using the generating signal lower bound $b_{\min}=0.4$.

\paragraph{TD-Lasso.}
We use the official R implementation of the top-down procedure \citep{chen2019causal}, with subset-search parameter $J=q$, supplied indegree bound $q=d$ given by an oracle, and cross-validated Lasso parent selection. The Lasso penalty is selected by ten-fold cross-validation using the one-standard-error rule.

\paragraph{Daskalakis Alg.~2.}
We use the published Python implementation of Algorithm~2 \citep{daskalakis2025condition}, with Lasso penalty $0.01$, supplied indegree bound $q=d$ given by an oracle, and coefficient threshold $b_{\min}/2=0.2$, using the generating lower bound $b_{\min}=0.4$.

\paragraph{GOLEM-EV.}
We use the official Python/TensorFlow equal-variance implementation on CPU \citep{ng2020role}, with $\lambda_1=0.02$, $\lambda_2=5$, learning rate $0.001$, and 100,000 optimizer updates. The official postprocessing routine thresholds edge weights at $0.3$ and removes edges until the result is acyclic.

\paragraph{Metrics.}
True positives, false positives, and false negatives are counted on directed edges. Exact recovery means equality of the entire directed adjacency matrices. F1 is $2\mathrm{TP}/(2\mathrm{TP}+\mathrm{FP}+\mathrm{FN})$, with value one when both edge sets are empty. Directed SHD is $\mathrm{FP}+\mathrm{FN}$, so reversing an edge contributes two errors. Metrics are computed for each repetition and then averaged. Standard errors are the sample standard deviations divided by the square root of the number of repetitions. The graph-recovery figures report means with error bars of one standard error across repetitions.

Figure~\ref{fig:main-accuracy-full} reports exact recovery, F1, and SHD for the method comparison in Figure~\ref{fig:all}(a)--(b). At $p$ = 200, $n$ = 500, HOST and HOST-CV/EBIC attain F1 $0.984$ and $0.985$, respectively, close to $0.991$ for TD-BSS, which has higher exact recovery and lower SHD. At $n$ = 100, HOST-CV/EBIC has higher F1 and lower SHD than HOST at both dimensions, while the two variants become similar as sample size increases. The recovery trends are similar to those in Figures~\ref{fig:host-scaling} and~\ref{fig:all}(a)--(b).

\begin{figure}[htb]\centering
\captionsetup[subfloat]{captionskip=2pt,farskip=0pt,nearskip=0pt}
\subfloat[$p$ = 50]{\includegraphics[width=.32\linewidth]{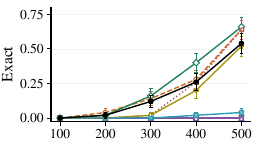}\label{fig:main-full-exact-p50}}\hfill
\subfloat[$p$ = 50]{\includegraphics[width=.32\linewidth]{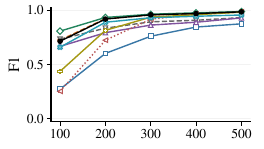}\label{fig:main-full-f1-p50}}\hfill
\subfloat[$p$ = 50]{\includegraphics[width=.32\linewidth]{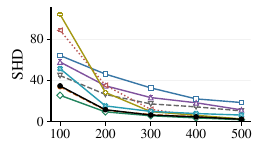}\label{fig:main-full-shd-p50}}\par
\subfloat[$p$ = 200]{\includegraphics[width=.32\linewidth]{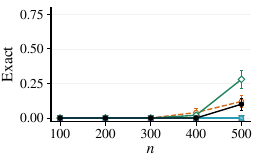}\label{fig:main-full-exact-p200}}\hfill
\subfloat[$p$ = 200]{\includegraphics[width=.32\linewidth]{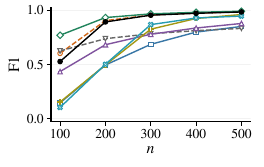}\label{fig:main-full-f1-p200}}\hfill
\subfloat[$p$ = 200]{\includegraphics[width=.32\linewidth]{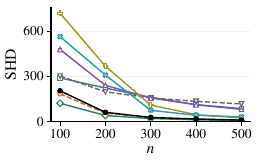}\label{fig:main-full-shd-p200}}\par
\includegraphics[width=\linewidth]{figures/host_main_recovery_legend.pdf}\par
\caption{Graph recovery for $d$ = 3. Panels (a)--(c) use $p$ = 50, and (d)--(f) use $p$ = 200. Columns show exact recovery, F1, and SHD. Points and error bars show means $\pm1$ standard error over 50 paired repetitions.}
\label{fig:main-accuracy-full}\end{figure}

\subsection{Runtime Experiments}\label{app:timing-protocol}
The runtime experiments in Figures~\ref{fig:all}(c)--(d) and~\ref{fig:computation-full} fix $n$ = 100 and use all combinations of $p$ = 10, 30, 50, 70, 100 and $d$ = 1, 2, 3, 4, 5, with 20 paired repetitions per setting. Each repetition uses the same input across methods. Daskalakis Alg.~1 is evaluated only at $p$ = 10, 30, 50. GOLEM-EV is excluded from this single-core CPU runtime comparison because its published scalability experiments use GPU acceleration \citep{ng2020role}. HOST, HOST-CV/EBIC, HLSM, TD-BSS, TD-BIC, and Daskalakis Alg.~1 use our compiled Cython/C++ implementations for both graph recovery and runtime measurement.

For TD-Lasso, graph recovery uses the official R code, whereas runtime is measured with our C++ implementation. For Daskalakis Alg.~2, graph recovery uses the published Python code, whereas runtime is measured with our Cython/C++ implementation. The timing implementations preserve the algorithms, supplied information, constants, and per-dataset tuning rules specified in Appendix~\ref{app:simulation-tuning}. Their numerical agreement with the graph-recovery implementations is checked before timing.

The runtime experiments use Intel Xeon Platinum 8592+ processors. Each fit uses a single physical CPU core and one numerical-library thread. Full-fit time is the wall-clock time for the full fitting procedure on each prepared input, including ordering, parent selection, final graph construction, method-internal transformations, and tuning. This includes CV and EBIC parameter selection for HOST-CV/EBIC and ten-fold cross-validation for TD-Lasso. HOST uses ordering and parent-threshold multipliers $0.5$ and $2$ without parameter search.

Figure~\ref{fig:computation-full} extends the two runtime sweeps in Figure~\ref{fig:all}(c)--(d) to all settings, varying $d$ at each fixed $p$ and $p$ at each fixed $d$. The full comparison shows runtime trends similar to those in Figure~\ref{fig:all}(c)--(d).

\begin{figure}[htb]\centering
\captionsetup[subfloat]{captionskip=2pt,farskip=0pt,nearskip=0pt}
\subfloat[$p$ = 10]{\includegraphics[width=.232727\linewidth]{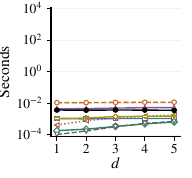}\label{fig:time-full-p10}}\hfill
\subfloat[$p$ = 30]{\includegraphics[width=.181818\linewidth]{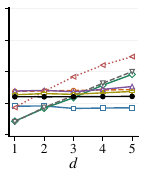}\label{fig:time-full-p30}}\hfill
\subfloat[$p$ = 50]{\includegraphics[width=.181818\linewidth]{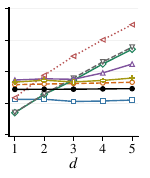}\label{fig:time-full-p50}}\hfill
\subfloat[$p$ = 70]{\includegraphics[width=.181818\linewidth]{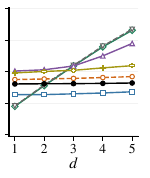}\label{fig:time-full-p70}}\hfill
\subfloat[$p$ = 100]{\includegraphics[width=.181818\linewidth]{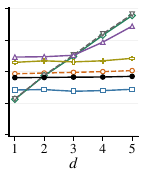}\label{fig:time-full-p100}}\par
\subfloat[$d$ = 1]{\includegraphics[width=.232727\linewidth]{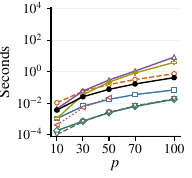}\label{fig:time-full-d1}}\hfill
\subfloat[$d$ = 2]{\includegraphics[width=.181818\linewidth]{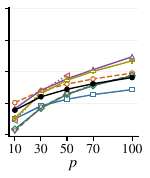}\label{fig:time-full-d2}}\hfill
\subfloat[$d$ = 3]{\includegraphics[width=.181818\linewidth]{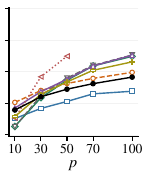}\label{fig:time-full-d3}}\hfill
\subfloat[$d$ = 4]{\includegraphics[width=.181818\linewidth]{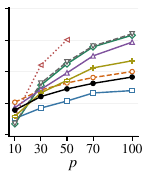}\label{fig:time-full-d4}}\hfill
\subfloat[$d$ = 5]{\includegraphics[width=.181818\linewidth]{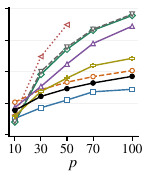}\label{fig:time-full-d5}}\par
\includegraphics[width=\linewidth]{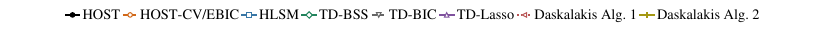}\par
\caption{Mean full-fit time (seconds, log scale) at $n$ = 100, including tuning, over 20 paired repetitions. Panels (a)--(e) vary $d$ at fixed $p$, and (f)--(j) vary $p$ at fixed $d$.}
\label{fig:computation-full}\end{figure}

\subsection{Additional Graph Families}\label{app:additional-graphs}
Figure~\ref{fig:additional-recovery} compares recovery on ER, hub, and bipartite DAGs with $p$ = 30, $n$ = 100, 200, 300, 400, 500, and 50 paired repetitions. It includes the same methods and recovery implementations as Figure~\ref{fig:all}(a)--(b), including Daskalakis Alg.~1 and both HOST variants. Coefficient and error distributions, node permutations, and tuning settings follow the main method comparison. Subset-search methods receive the supplied indegree bound $q=d$, where $d$ is the realized maximum indegree of each graph.

For ordered Erd\H{o}s--R\'enyi DAGs, each potential edge from an earlier to a later node is sampled independently with probability $2/(p-1)$, giving expected edge count $p$. For hub DAGs, three nodes with at least six earlier nodes are selected as hubs, each receiving six uniformly sampled incoming edges. Other nodes receive independently sampled incoming edges with probability $0.02$, conditioned on having at most five parents. The maximum indegree is therefore exactly six. Bipartite DAGs have two groups of 15 nodes. Edges are sampled independently from the first group to the second with probability $2/15$, giving expected edge count 30, with no edges within either group.

\begin{figure}[htb]\centering
\includegraphics[width=\linewidth]{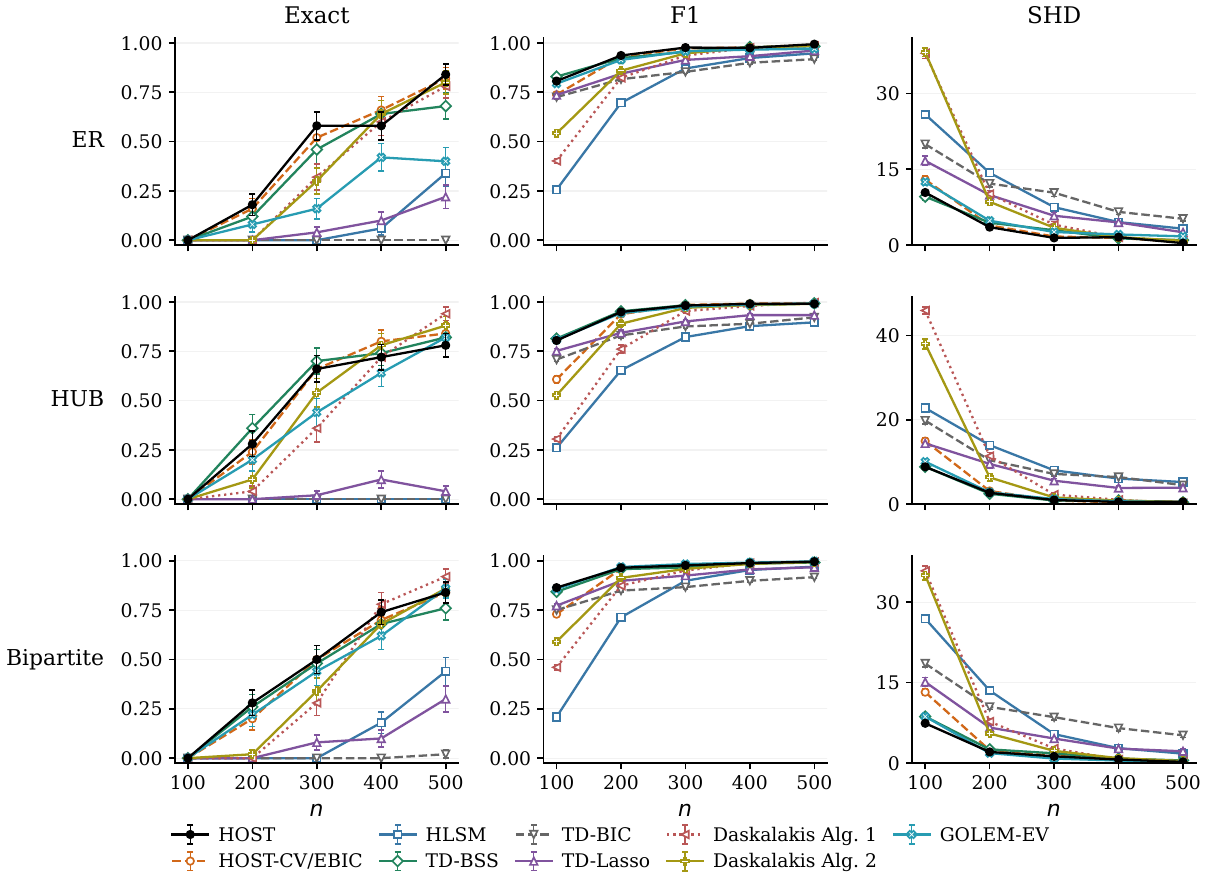}
\caption{Graph recovery for ER, hub, and bipartite DAGs in the first, second, and third rows, respectively, with $p$ = 30. Columns show exact recovery, F1, and SHD. Points and error bars show means $\pm1$ standard error across 50 paired repetitions.}
\label{fig:additional-recovery}\end{figure}

Figure~\ref{fig:additional-recovery} shows that relative performance varies across graph families and recovery metrics. At $n$ = 500, HOST achieves the highest exact recovery on ER graphs, whereas Daskalakis Alg.~1 achieves higher exact recovery on hub and bipartite graphs. The two HOST variants become similar in F1 and SHD as sample size increases, and both provide competitive graph recovery across all three graph families.

\subsection{Ordering Recovery Comparison with Subset Search Methods}\label{app:ordering-bound}
We compare the ordering steps of HOST and TD-BSS \citep{gao2022optimal} at $p$ = 50, $d$ = 2, 3, 4, and $n$ = 100, 200, 300, 400, 500, with 50 paired repetitions. Each graph contains the chain $1\to2\to\cdots\to p$, which makes its compatible ordering unique. For $j\ge2$, node $j$ has exactly $\min(d,j-1)$ parents, including $j-1$, with any additional parents sampled uniformly from $\{1,\ldots,j-2\}$. Coefficients, error variances, and node permutations follow Appendix~\ref{app:simulation}. HOST uses the fixed ordering penalty in Appendix~\ref{app:simulation-tuning}, while TD-BSS receives supplied indegree bounds $q$ = 2, 3, 4, 5, 6. We report the normalized Kendall distance $K/\binom p2$, where $K$ counts node pairs ordered differently in the estimated and true orderings. We also compare ordering runtime.

\begin{figure}[htbp]\centering
\includegraphics[width=\linewidth]{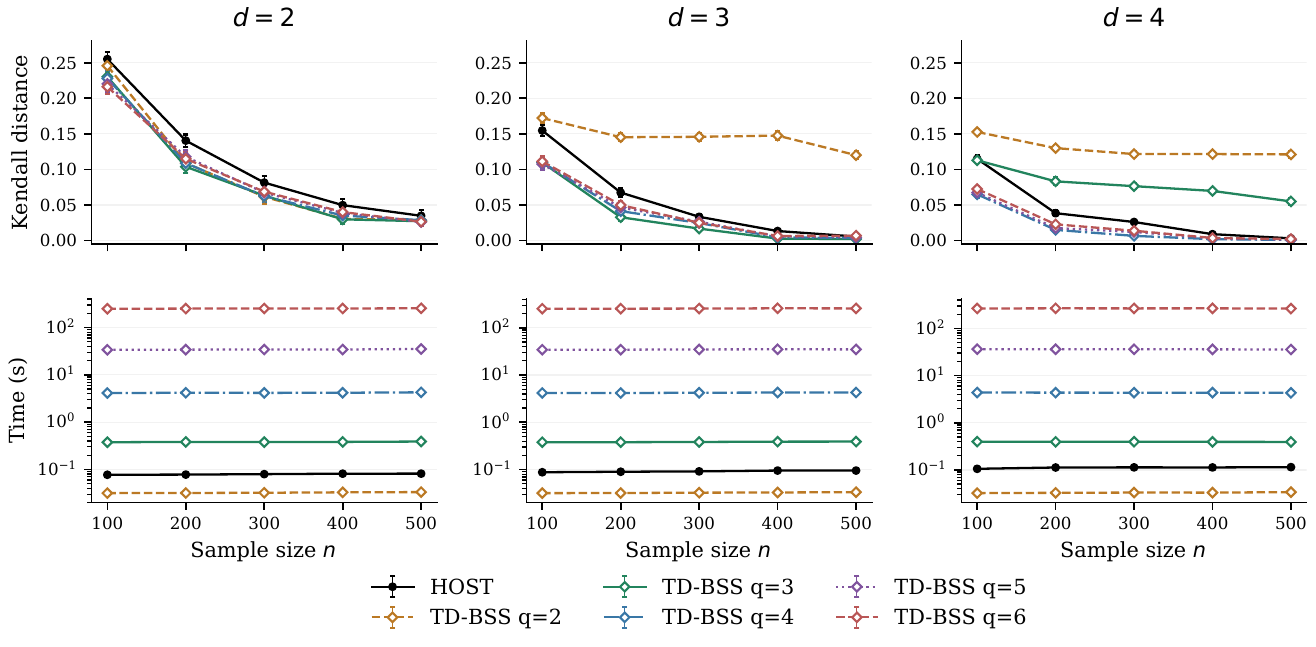}
\caption{Ordering recovery and computational cost at $p$ = 50. Rows show Kendall distance and runtime, and columns show $d$ = 2, 3, 4. Points and error bars show means $\pm1$ standard error across 50 paired repetitions.}
\label{fig:ordering-bounds}\end{figure}

Figure~\ref{fig:ordering-bounds} shows that HOST achieves decreasing Kendall distance as sample size increases without receiving a supplied indegree bound, while TD-BSS generally achieves smaller Kendall distance when $q\ge d$. Supplying $q<d$ can substantially degrade TD-BSS recovery. When $d$ is unknown, choosing a larger $q$ reduces the risk of under-specification but sharply increases the cost of subset search. This comparison illustrates the computational advantage of ordering recovery without a supplied indegree bound.

\subsection{Parent Recovery Comparison with Lasso}\label{app:oracle-parent}
We isolate parent recovery by supplying the true ordering to HOST and Lasso. The graph construction follows Appendix~\ref{app:ordering-bound}, with $p$ = 50, $d$ = 2, 3, 4, $n$ = 100, 200, 300, 400, 500, and 50 paired repetitions. Lasso fits each node on all earlier variables using penalty $\lambda=c\sqrt{\log p/n}$ and selects variables with nonzero coefficients as parents. HOST uses parent threshold $\tau_j=c\sqrt{\log p/n}$ and $\delta=0.8$. The two-fold ordering fits used to construct the parent-recovery guide follow the true ordering and use a fixed penalty multiplier of $0.5$. HOST reuses its main-comparison implementation, and Lasso uses the same coordinate-descent solver.

To remove differences due to tuning parameter selection, each method uses oracle tuning over $c\in\{0.5,1,2,3,4,5\}$. In each repetition, one multiplier for the entire graph is chosen to minimize SHD against the true graph, with ties resolved in favor of the larger multiplier. SHD, precision, and recall are all computed from that same selected graph. Oracle tuning provides a reference for recovery achievable within this candidate pool.

\begin{figure}[htbp]\centering
\includegraphics[width=\linewidth]{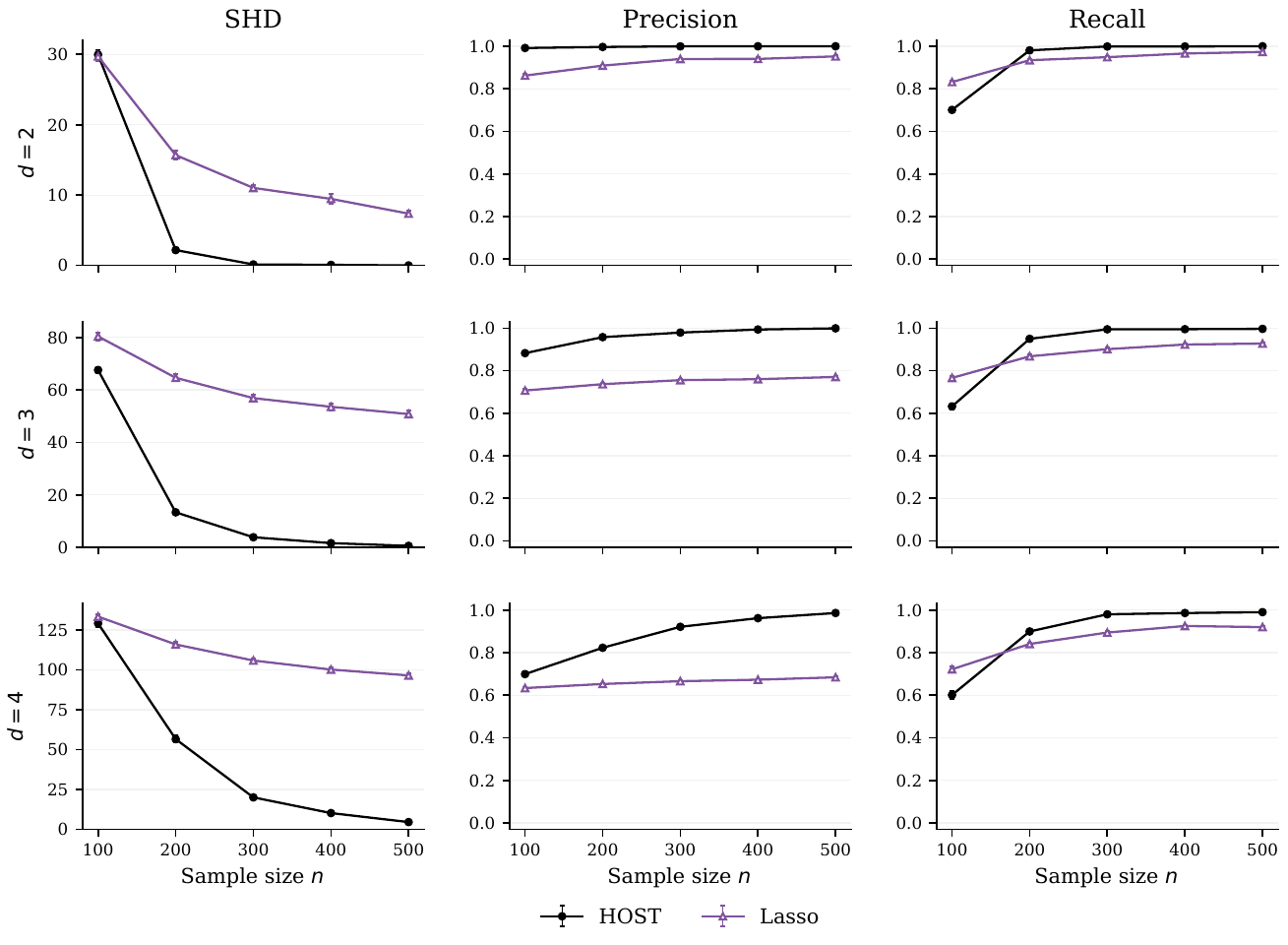}
\caption{Parent recovery with the true ordering and oracle tuning at $p$ = 50. Rows correspond to $d$ = 2, 3, 4, and columns show SHD, precision, and recall. Each method selects $c\in\{0.5,1,2,3,4,5\}$ by minimizing SHD separately in each repetition. Points and error bars show means $\pm1$ standard error across 50 paired repetitions.}
\label{fig:oracle-parent-50}\end{figure}

Figure~\ref{fig:oracle-parent-50} shows that HOST generally achieves lower SHD than Lasso as sample size increases, with a larger gap at larger maximum indegrees, although Lasso performs better in some small-sample settings. The difference remains after each method receives its best multiplier in the candidate pool. Increasing $d$ introduces more direct and indirect paths and can make covariance conditions for Lasso support recovery, including the irrepresentable condition \citep{zhao2006model}, harder to satisfy. The observed pattern is consistent with this limitation of Lasso. HOST instead reconstructs direct effects without requiring exact Lasso support recovery or irrepresentability, as established in Theorem~\ref{thm:parent-recovery}.

\section{Real-Data Experiments}\label{app:realdata}\label{sec:realdata}
We compare HOST with six methods on the Sachs sample available from the \texttt{bnlearn} website \citep{sachs2005causal} ($n$ = 853, $p$ = 11) and the uniform reference experiment from the Light Tunnel dataset \citep{gamella2025causalchambers} ($n$ = 10,000, $p$ = 20). Each dataset is column-centered using all observations before fitting. The supplied reference DAGs contain 17 edges for Sachs and 39 for Light Tunnel and are used only to evaluate directed F1 and SHD.

We use HOST, HLSM, TD-BSS, TD-BIC, TD-Lasso, and Daskalakis Algs.~1 and~2, following the implementations in Appendix~\ref{app:simulation-tuning}. GOLEM-EV is excluded from this comparison. All tuning multipliers range over $\mathcal C=\{2^{-7},2^{-6},\ldots,2^7\}$, and supplied indegree bounds range over $Q=\{1,2,4,6,8\}$.

For HOST, pooled two-fold CV selects $c_{\mathrm{ord}}\in\mathcal C$ and the common ordering penalty $\lambda_{\mathrm{ord}}=c_{\mathrm{ord}}\sqrt{\log p/\lfloor n/2\rfloor}$. After recovering the ordering with this penalty, HOST processes nodes in that ordering and selects $c_j\in\mathcal C$ separately for each node by EBIC, giving parent threshold $\tau_j=c_j\sqrt{\log p/n}$. The CV and EBIC criteria follow Appendix~\ref{app:ebic-proof}.

For each comparison method, we construct a candidate DAG $G$ for every parameter combination in Table~\ref{tab:realdata-settings} and select the graph and parameters minimizing Gaussian BIC \citep{biza2020tuning},
\[
\mathrm{BIC}(G)=n\sum_{j=1}^{p}\log\!\left(\frac{\mathrm{RSS}_j(G)}{n}\right)
+|E(G)|\log n.
\]
Here $E(G)$ is the edge set of $G$, and $\mathrm{RSS}_j(G)$ is the residual sum of squares from OLS regression of node $j$ on its parents in $G$, without an intercept on the centered data.

\begin{table}[htbp]
\centering\small
\caption{Real-data tuning parameters and candidate construction.}
\label{tab:realdata-settings}
\begin{tabularx}{\linewidth}{@{}lX@{}}
\toprule
Method & Parameters and selection procedure \\
\midrule
HOST & Select $c_{\mathrm{ord}}\in\mathcal C$ by pooled two-fold CV, then $c_j\in\mathcal C$ for each node by EBIC. \\
HLSM & For each $c\in\mathcal C$, fit with Lasso penalty $\lambda=c\sqrt{\log p/n}$. \\
TD-BSS & For every $(q,c)\in Q\times\mathcal C$, fit with supplied indegree bound $q$ and parent-selection threshold $\eta=c\log p/n$. \\
TD-BIC & For each $q\in Q$, select each parent set by minimizing nodewise BIC over subsets of at most $q$ earlier nodes. \\
TD-Lasso & For each $q\in Q$, recover the ordering and select nodewise Lasso penalties by ten-fold cross-validation using the one-standard-error rule. \\
Daskalakis Alg.~1 & For every $(q,c_b)\in Q\times\mathcal C$, fit with supplied indegree bound $q$ and $b_{\min}=c_b\sqrt{\log p/n}$. \\
Daskalakis Alg.~2 & For every $(q,c_b,c_\lambda)\in Q\times\mathcal C^2$, fit with supplied indegree bound $q$, $b_{\min}=c_b\sqrt{\log p/n}$, and Lasso penalty $\lambda=c_\lambda\sqrt{\log p/n}$. \\
\bottomrule
\end{tabularx}
\end{table}

\begin{table}[htbp]
\centering\small
\caption{Directed F1 and SHD on the full Sachs and Light Tunnel datasets. Bold entries indicate the highest F1 or lowest SHD in each column.}
\label{tab:realdata-results}
\begin{tabular}{@{}lrrrr@{}}
\toprule
& \multicolumn{2}{c}{Sachs} & \multicolumn{2}{c}{Light Tunnel} \\
\cmidrule(lr){2-3}\cmidrule(l){4-5}
Method & F1 & SHD & F1 & SHD \\
\midrule
HOST & \textbf{0.385} & \textbf{16} & \textbf{0.518} & \textbf{41} \\
HLSM & 0.276 & 21 & 0.477 & 57 \\
TD-BSS & 0.303 & 23 & 0.344 & 80 \\
TD-BIC & \textbf{0.385} & \textbf{16} & 0.506 & \textbf{41} \\
TD-Lasso & 0.200 & \textbf{16} & 0.517 & 43 \\
Daskalakis Alg.~1 & 0.357 & 18 & 0.411 & 66 \\
Daskalakis Alg.~2 & 0.370 & 17 & 0.472 & 47 \\
\bottomrule
\end{tabular}
\end{table}

Table~\ref{tab:realdata-results} shows that HOST and TD-BIC tie for the highest Sachs F1 ($0.385$) and, together with TD-Lasso, attain the lowest SHD ($16$). On Light Tunnel, HOST attains F1 $0.518$, close to $0.517$ for TD-Lasso, and ties TD-BIC for the lowest SHD ($41$). These results indicate competitive recovery by HOST on both datasets among methods developed under similar modeling assumptions.

\end{document}